\documentclass[twoside]{article}
\usepackage[accepted]{aistats2025}

\usepackage[round]{natbib}

\newcommand{\R}{\mathbb{R}}
\newcommand{\norm}[1]{\left\| {#1} \right\|}
\newcommand{\argmin}[1]{\underset{#1}{\text{argmin}}\;}
\newcommand{\argmax}[1]{\underset{#1}{\text{argmax}}\;}

\usepackage{hyperref}
\usepackage{xcolor}
\usepackage{subcaption}
\usepackage[linesnumbered, boxed]{algorithm2e}
\usepackage{graphicx}
\usepackage{amsmath}
\usepackage{amsthm}
\usepackage{amsfonts}
\usepackage{dsfont}
\usepackage{bm}
\usepackage{enumitem}
\usepackage{pifont}
\usepackage{cleveref}
\usepackage{aliascnt}

\newtheorem{assumption}{\textbf{H}\hspace{-3pt}}
\crefformat{assumption}{{\textbf{H}}#2#1#3}
\newtheorem{theorem}{Theorem}
\newtheorem{lemma}{Lemma}

\newaliascnt{proposition}{theorem}
\newaliascnt{definition}{theorem}
\newaliascnt{corollary}{theorem}
\newaliascnt{remark}{theorem}

\aliascntresetthe{proposition}
\aliascntresetthe{definition}
\aliascntresetthe{corollary}
\aliascntresetthe{remark}

\crefname{proposition}{proposition}{propositions}
\Crefname{Proposition}{Proposition}{Propositions}
\crefname{definition}{definition}{definitions}
\Crefname{Definition}{Definition}{Definitions}
\crefname{corollary}{corollary}{corollaries}
\Crefname{Corollary}{Corollary}{Corollaries}
\crefname{example}{example}{examples}
\Crefname{Example}{Example}{Examples}
\crefname{theorem}{theorem}{Theorems}
\Crefname{Theorem}{Theorem}{Theorems}
\crefname{remark}{remark}{remarks}
\Crefname{Remark}{Remark}{Remarks}
\crefname{figure}{figure}{figures}
\Crefname{Figure}{Figure}{Figures}

\begin{document}
\runningtitle{Personalized CDL of Physiological Time Series}
\runningauthor{A. Roques, S. Gruffaz, K. Kim, A. O. Durmus, L. Oudre}

\twocolumn[
    \aistatstitle{Personalized Convolutional Dictionary Learning of\\Physiological Time Series}
    \aistatsauthor{ Axel Roques \And Samuel Gruffaz \And Kyurae Kim }
    \aistatsaddress{  Centre Borelli, \\ Laboratoire GBCM, Thales AVS \And Centre Borelli,\\  ENS Paris-Saclay  \And University of Pennsylvania } 
    \aistatsauthor{  Alain O. ~Durmus \And Laurent ~Oudre }
    \aistatsaddress{ CMAP, CNRS, Ecole polytechnique \And Centre Borelli,  ENS Paris-Saclay } 
]

\begin{abstract}
Human physiological signals tend to exhibit both global and local structures: the former are shared across a population, while the latter reflect inter-individual variability. 
For instance, kinetic measurements of the gait cycle during locomotion present common characteristics, although idiosyncrasies may be observed due to biomechanical disposition or pathology.
To better represent datasets with local-global structure, this work extends Convolutional Dictionary Learning (CDL), a popular method for learning interpretable representations, or \textit{dictionaries}, of time-series data.
In particular, we propose Personalized CDL (PerCDL), in which a \textit{local} dictionary models local information as a personalized spatiotemporal transformation of a  \textit{global} dictionary.
The transformation is learnable and can combine operations such as time warping and rotation.
Formal computational and statistical guarantees for PerCDL are provided and its effectiveness on synthetic and real human locomotion data is demonstrated.
\end{abstract}

%%%%%%%%%%%%%%%%%%%%%%%%%%%%%%%%%%%%%%%%%%%%%%%%%%%%%%%%%%%%
\section{INTRODUCTION}
\label{sec:introduction}
Dictionary learning (DL; \citealp{kreutz2003dictionary,tovsic2011dictionary}) is an unsupervised representation learning method for decomposing data into two components: a \textit{dictionary}, \textit{i.e.}, a collection of reference signals, each called an \textit{atom}, and a set of mixing coefficients. 
Among the numerous variants of DL, Convolutional Dictionary Learning (CDL; \citealp{garcia2018convolutional}) assumes that the dataset can be synthesized with atoms of relatively smaller length repeated at different positions. 
CDL is a well-established method to represent various time-series modalities (\textit{e.g.}, audio signals; \citealp{grosse2007shift}) that excels at learning interpretable representations of human physiology~\citep{dupre2018multivariate, power2023using,chen2023parametric}: many human activities are structured and tend to exhibit periodic patterns, which can be captured by CDL.
During locomotion for instance---our illustrative example throughout this text---CDL provides a principled way to statistically infer the gait cycles directly from the data~\citep{whittle2014gait}, thereby circumventing the need for experts to craft a dictionary themselves (typically for template matching, \textit{cf.},~\citealp{ying2007automatic, mico2016novel, oudre2018template, dot2020non, voisard2023automatic}).

\begin{figure*}
\centering
  \includegraphics[width=\textwidth]{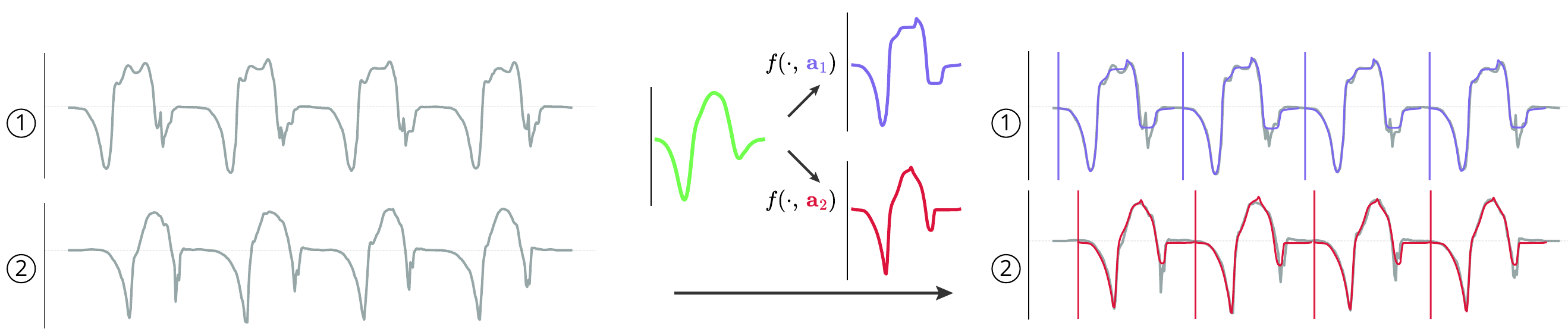}
\caption{
    \textbf{Illustration of PerCDL for the Analysis of Human Locomotion Data.}
    (left) The foot kinematic signals of two individuals \ding{172} and \ding{173} present a repeated structure called a gait cycle. 
    (middle) PerCDL learns the shared structures and subject-specific variability: the common shape (green) is ``personalized'' using a transformation function $f$ (parameterized by the personalization parameters $\bm{a}_1$ and $\bm{a}_2$). 
    (right) PerCDL successfully identifies all gait cycles and signal-specific shapes (blue and red), resulting in an accurate reconstruction.
}
\label{fig:illustration}
\vspace{-2ex}
\end{figure*}
%%%%%

Nonetheless, vanilla CDL is unable to handle a key characteristic of human physiological data, namely their \textit{local-global} structure. At the population-level, \textit{global}, shared, patterns may emerge due to similar anatomical traits or common neuronal circuits. At the individual's scale, \textit{local}, personal variations around these commonalities may exist due to various idiosyncrasies, differences in perceptual-motor style or pathology \citep{mantilla2020motor, vidal2021perceptual, webster2019principles}. 
Naive application of CDL to this type of data (hereafter referred to as PopCDL for ``population-level CDL'') will yield a globally shared \textit{population dictionary} that disregards individual-level structure, treating it as noise.
Conversely, applying CDL to each individual independently (hereafter referred to as IndCDL for ``individual-level CDL'') precludes the learning of the shared structures in the data, and can be unreliable under the presence of noise.
Despite efforts to address these challenges, the global-local structure has not been fully leveraged to achieve meaningful representations of human physiology (see \Cref{sec:related_works}).

To address these issues, we propose Personalized CDL (PerCDL), which aims to infer both global (population-level structures) and local (individual-level peculiarities) patterns from a collection of time series.
The reconstruction process of PerCDL is illustrated in \Cref{fig:illustration}.
Briefly, PerCDL learns individual-specific personalized atoms from a set of common atoms shared across the population.
Personalization is achieved through spatiotemporal transformation of the common atoms using Time Warping (TW;~\citealp{muller2007information}).
Specifically, we adapt the shape reparameterization method of \citet{celledoni2023deep} to the CDL context.
Unlike previous works that also investigated the use of TW in DL \citep{xu2023generalized, yazdi2018time}, our approach does not require the manual selection of a set of increasing functions \citep{xu2023generalized} and is computationally more efficient than that of \citet{yazdi2018time}, making it more robust, scalable, and usable.

While the proposed framework is general in nature, particular consideration will be given to the analysis of neurophysiological signals (\textit{e.g.}, heart rate, breathing, walking, etc).
Specifically, the analysis of gait.
To better model such signals, a common practice in CDL has been to impose non-overlapping atoms to reflect the underlying physiological origin of these events~\citep{germain2024persistence, zhu2016matrix, schafer2022motiflets, torkamani2017survey, jas2017learning, dupre2018multivariate, power2023using}.
The no-overlap assumption will be implicit in the remainder of this text. 
Besides being physiologically relevant, it is also computationally convenient enabling closed-form updates of the common dictionary (the \textit{exact} solution of the $l_0$ sparse coding problem;~\citealp{Charles24}) and parallelism at each step. 
Furthermore, this constraint allows us to conduct an analysis of the convergence rate of the associated Maximum Likelihood Estimator (MLE). 
For applications where the the no-overlap assumption is no longer realistic, any alternative optimisation method that relaxes this assumption~\citep{song2018spike, tolooshams2020convolutional} can be used instead.

%A summary of our contributions is provided below:
\begin{itemize}[leftmargin=3ex,noitemsep,topsep=-8pt]
    \item We propose PerCDL\footnote{\url{https://github.com/axelroques/PerCDL}.}, a framework for learning interpretable representations of time-series data that highlight personal deviations from common structures (\Cref{sec:PerCDL_methodo}).
    \item We present a meta-algorithm for solving the PerCDL problem (\Cref{sec:meta-algorithm}). The algorithm is modular in design, offering the flexibility to combine various components from conventional CDL \citep{garcia2018convolutional, Charles24}. A federated learning, fully parallelizable variant specific to physiological time series representation is also introduced (\Cref{sec:federated}). Its complexity is $\mathcal{O}(N\log(N))$ per processor, where $N$ is the number of samples per time series.
    \item We provide statistical guarantees for PerCDL (\Cref{sec:theory}, \Cref{theorem:convergence1}): the estimated dictionary converges towards the true common dictionary at a rate of $\mathcal{O}(1/\sqrt{Sp})$, with $S$ the number of individuals and $p$ the number of pattern observations per individual.
    This improves over the naive scheme, IndCDL, which achieves a rate of $\mathcal{O}(1/\sqrt{p})$.
    \item We evaluate PerCDL on synthetic and real-world human physiology data, namely locomotion and ECG signals (\Cref{sec:experiments}). The results demonstrate that PerCDL reliably extracts interpretable global and local motifs while being robust to noise.  
\end{itemize}

\IncMargin{2em}
\SetKwComment{Comment}{// }{}
\SetKw{Input}{Inputs}
\SetKw{Output}{Outputs}

\begin{algorithm*}[ht]
\caption{
\textbf{Meta-algorithm to Solve the PerCDL Problem in \Cref{eq:PerCDL}}. 
$\hat{\bm{\Phi}}$ is the personalized dictionary obtained from the transformation of the atoms in the common dictionary $\bm{\Phi}$.\vspace{-6ex}
}
\label{alg:PerCDL}
\Input{$\bm{X} \in \R^{S \times N\times P}$, $f:(\mathbb{R}^P)^L\times \mathbb{R}^M \mapsto (\mathbb{R}^P)^L $} \\ 
\Output{$\bm{\Phi} \in \R^{K \times L\times P}$, $\bm{A} \in \R^{S \times K \times M}$, $\bm{Z} \in \R^{S \times K \times N-L+1}$
}

\setcounter{AlgoLine}{0}

$\bm{\Phi}, \bm{A}, \bm{Z} \gets \text{setInitialValues}()$ \Comment{Initialization}

%\BlankLine
\Comment{Obtain initial estimate for the common dictionary}
\While{$i \leq n_{\mathrm{init}}$}{
    $\bm{Z} \gets \text{CSC}\left( \bm{X}, \bm{Z}, \bm{\Phi} \right)$: for any $s\in[S]$, $\argmin{\bm{Z}_s \geq 0} \norm{
	    \bm{x}^s - \sum_{k=1}^K \bm{z}^s_k * \bm{\phi}_k
    }_2^2 
    + \lambda \norm{\bm{Z}}_0$\;
    $\bm{\Phi} \gets \text{CDU}\left( \bm{X}, \bm{Z}, \bm{\Phi} \right)$: $\argmin{\bm{\Phi}: \norm{\bm{\phi}}_2 = 1} 
    \sum_{s=1}^{S} \norm{ \bm{x}^s - \sum_{k=1}^{K} \bm{z}^s_k * \bm{\phi}_k }_2^2$\;
}

\BlankLine
\Comment{Personalization}
\While{$i \leq n_{\mathrm{perso}}$}{
    $\bm{A} \gets \text{IPU}\left( \bm{X}, \bm{Z}, \bm{\Phi},\bm{A}, f \right)$: for any $s\in[S]$, $\argmin{\bm{A}} 
     \norm{
    	\bm{x}^s - \sum_{k=1}^K \bm{z}^s_k * f\left( \bm{\phi}_{k}, \bm{a}_{k}^{s} \right)
    }_2^2$, \;
    \,$\bm{Z} \gets \text{CSC}\left( \bm{X}, \bm{Z}, \hat{\bm{\Phi}} \right)$\;
    $\bm{\Phi} \gets \text{PerCDU}\left( \bm{X}, \bm{Z}, \bm{\Phi},\bm{A},f \right):$ 
    $\argmin{\bm{\Phi}: \norm{\bm{\phi}_k}_2 = 1}\sum_{s=1}^S \norm{
        \bm{x}^s - \sum_{k=1}^K \bm{z}^s_k * f\left( \bm{\phi}_{k}, \bm{a}_{k}^{s} \right)
    }_2^2$\;
}
\end{algorithm*}
%\vspace{-2ex}

%%%%%%%%%%%%%%%%%%%%%%%%%%%%%%%%%%%%%%%%%%%%%%%%%%%%%%%%%%%%

\section{PERSONALIZED CDL}
\label{sec:methodology}

\textbf{Notations.}
The integer range $\left\{ k, \ldots, l \right\} \subset \mathcal{P}\left( \mathbb{Z} \right)$ is written $\left[ k\, : \, l\right]$, and by extension $\left[ l \right] = \left[ 1\, : \, l \right]$, with $k, \, l \in \mathbb{N}$. 
For any vector $\bm{u} = \left( u_t \right)_{t \in [T]}$, we define its $l_2$-norm by $\norm{\bm{u}}_2 = \sqrt{\sum_{t=1}^T \| u_t \|^2}$ and its $l_0$-norm by $\norm{\bm{u}}_0 = \sum_{t=1}^T \mathrm{1}_{\neq 0} \left( u_t \right)$. 
We denote by $\mathcal{O}_P=\{O \in \mathbb{R}^{P\times P}: O^\top O = I_P \}$ and $C^1([0,1])$ the space of continuously differentiable functions $g:[0,1]\to [0,1]$.

%%%%%%%%%%%%%%%%%%%%%%%%%
\vspace{-1ex}
\subsection{Convolutional Dictionary Learning}
Consider a time-series dataset $\bm{X} \in \R^{S \times N \times P}$ of $S$ $P$-dimensional time series $\left( \bm{x}^s \right)_{s \in \left[ S \right]}$ of equal size $N$.
In the context of human physiological signals, \(S\) may correspond to the number of unique individuals, \(P\) may refer to the number of sensor features, and $N$ to the maximum number of samples per individual.
Given a shared dictionary $\bm{\Phi}=(\bm{\phi}_k)_{k \in [K]} \in \R^{K \times L\times P}$ of $K$ common patterns of length $L$ ($L < N$), their respective weights, and activations $\bm{Z}=(\bm{z}^s_k)_{s, k \in [S] \times [K]}$ in $\R^{S \times K \times (N-L+1)}$, the CDL problem is stated as:
{%
\setlength{\belowdisplayskip}{0.5ex} \setlength{\belowdisplayshortskip}{0.5ex}
\setlength{\abovedisplayskip}{0.5ex} \setlength{\abovedisplayshortskip}{0.5ex}
\begin{equation}
	\argmin{\bm{Z} \geq 0, \; \bm{\Phi}: \norm{\bm{\phi}_k}_2 = 1} 
	\sum_{s=1}^S \norm{\textstyle
        \bm{x}^{s} - \sum_{k=1}^K \bm{z}^s_k * \bm{\phi}_k
    }^2_2
 + \lambda \norm{\bm{Z}}_0,
 \label{eq:cdl}
\end{equation}
}%
where $*$ denotes the multivariate linear convolution such that, for two time series $\bm{z}=(z_i)_{i\in[N-L+1]}\in \mathbb{R}^{N-L+1}$ and $\bm{y}=(y_j)_{j\in[L]}\in (\mathbb{R}^{P})^L$,
{%
\setlength{\belowdisplayskip}{0.5ex} \setlength{\belowdisplayshortskip}{0.5ex}
\setlength{\abovedisplayskip}{0.5ex} \setlength{\abovedisplayshortskip}{0.5ex}
\[
    \bm{z} * \bm{y} = \left({\textstyle\sum_{l=1}^L} y_l \, z_{m-l+1}\right)_{m\in [N]} \in (\mathbb{R}^P)^N.
\]
}%
$\lambda$ is a penalty term that balances the sparsity of the activation matrix and the reconstruction error. 
Increasing $\lambda$ results in sparser activations and yields more interpretable atoms but often at the cost of an increased reconstruction error.

This problem formulation does not explicitly consider the local-global structure of physiological signals: one can either learn a dictionary for each individual by running CDL \(S\)-times (IndCDL) or, alternatively, learn a single dictionary that is shared across the whole population (PopCDL).

\vspace{-1ex}
\subsection{Personalized Convolutional Dictionary Learning}
\vspace{-1ex}
\label{sec:PerCDL_methodo}

To accommodate inter-individual variability within the CDL framework, we propose to learn population-level atoms along with individual-level atoms. 
Each personalized atom \(\hat{\bm{\phi}}^s_k\) is derived from a common atom \(\bm{\phi}_k\) in the global dictionary \(\bm{\Phi}\) via a transformation function.
More formally, denoting a parameterized transformation as $f:(\R^P)^L\times \Theta\to (\R^P)^L$, where $\Theta \subset \R^M$ is the set of personalization parameters, the personalized version of the $k^{th}$ global atom $\bm{\phi}_k$ for the $s^{th}$ time series is expressed as, for any $(s, k) \in [S] \times [K]$,
{%
\setlength{\belowdisplayskip}{0.5ex} \setlength{\belowdisplayshortskip}{0.5ex}
\setlength{\abovedisplayskip}{0.5ex} \setlength{\abovedisplayshortskip}{0.5ex}
\begin{equation}
\label{eq:personalized_dictionnary}
	\hat{\bm{\phi}}_{k}^{s} = f\left( \bm{\phi}_{k}, \bm{a}_{k}^{s} \right)
\end{equation}
}%
with $\bm{a}^s_k \in \Theta$, the personalization parameters. 
This representation strictly generalizes both IndCDL and PopCDL since they can be obtained by setting $f(\bm{\phi}_k,\bm{a}_k^s)=\bm{a}_k^s\in (\mathbb{R}^d)^L$ or $f(\phi_k,a_k^s)=\phi_k$ respectively.
The choice of a suitable transformation function $f$ is discussed in \Cref{sec:transformation_function}.

With $\bm{A}=(\bm{a}_{k}^{s})_{s, k \in [S] \times [K]}\in \R^{S\times K\times M}$, the matrix of all personalization parameters, PerCDL solves
{%
\setlength{\belowdisplayskip}{0ex} \setlength{\belowdisplayshortskip}{0ex}
\setlength{\abovedisplayskip}{0ex} \setlength{\abovedisplayshortskip}{0ex}
\begin{equation}
	\label{eq:PerCDL}
	\argmin{\bm{Z} \geq 0, \, \bm{\Phi}: \norm{\bm{\phi}_k}_2 = 1, \, \bm{A}} 
	\sum_{s=1}^S \norm{\textstyle
        x^s-\sum_{k=1}^K \bm{z}^s_k * \hat{\phi}_{k}^{s}
    }_2^2 
  + \lambda \norm{\bm{Z}}_0.
\end{equation}
}%
In brief, the common structures of a dataset $\bm{X}$ are identified in a dictionary $\bm{\Phi}$, while time series-specific variations of these global features are captured through the parameters matrix $\bm{A}$.
Information about the occurrence and magnitude of these personalized motifs are contained in the activation matrix $\bm{Z}$.
\Cref{fig:illustration} illustrates the principles of PerCDL.

%%%%%%%%%%%%%%%%%%%%%%%%%%%%%%%%%%%%%%%%%%%%%%%%%%%%%%%%%%%%

\subsection{A Meta-Algorithm}
\label{sec:meta-algorithm}
This section introduces a general meta-algorithm for solving the PerCDL problem in \Cref{eq:PerCDL} (\Cref{alg:PerCDL}).
The algorithm comprises two steps: \ding{182} Obtain an initial estimate of the common dictionary \(\bm{\Phi}\); \ding{183} Refine \(\bm{\Phi}\) to yield the personalized dictionary \(\bm{\hat{\Phi}}\).

\vspace{-1ex}
\paragraph{\ding{182} Initial Estimates for \(\bm{\Phi}\) and \(\bm{Z}\) (Lines 2-5).} 
The classical CDL problem (\Cref{eq:cdl}) is first solved to obtain initial estimates for the common dictionary and the associated activations. 
Any off-the-shelf CDL algorithm can be used, such as a block-coordinate descent schemes, which alternate between a convolutional sparse coding step (CSC) and a convolutional dictionary update step (CDU).
Given a fixed dictionary, the CSC step finds the best activations $\bm{Z}$ subject to a sparsity-promoting penalization.
In this work, the \(l_0\)-penalty method proposed in \citet{Charles24} is adopted as it is both robust and exact for non-overlapping atoms. 
Other approximate $l_1$-penalized methods that relax this assumption could be used~\citep{kavukcuoglu2010learning, boyd2011distributed, moreau2018dicod}.
Given a fixed set of activations $\bm{Z}$, the CDU step finds the dictionary $\bm{\Phi}$ that minimizes the $l_2$-reconstruction error.
Once again, any dictionary update scheme can be used~\citep{garcia2018convolutional}. 

\vspace{-1ex}
\paragraph{\ding{183} Personalization (Lines 6-10).} 
Once sensible initial estimates for $\bm{\Phi}$ and $\bm{Z}$ have been found, the personalization parameters $\bm{A}$ are learned (Line 7), while the activations $\bm{Z}$ (Line 8) and the common dictionary $\bm{\Phi}$ (Line 9) are refined in a block-coordinate descent scheme.
In the individual parameters update (IPU) step, any constrained optimization algorithm can be used to obtain an approximate solution, such as projected gradient descent~\citep{bubeck2015convex}.
This work uses gradient descent with the Polyak stepsize~\citep{hazan2019revisiting}, where the infimum of the objective is substituted with its lower bound, $0$~\citep{loizou2021stochastic}.
This can be run in parallel on $S$ processors.
When updating the dictionary given \(\bm{A}\) and \(\bm{Z}\) (PerCDU step), the transformation \(f\) may introduce non-linearities that prohibit the use of conventional CDU update schemes, a problem shared with parametric DL methods~\citep{ataee2010parametric,turquais2018parabolic,chen2023parametric}.
In fortunate cases where \(f\) is linear in $\bm{\phi}$, conventional CDU updates can be exploited.

\vspace{-1ex}
\paragraph{Federated Learning Implementation.} 
Under the no-overlap assumption, the CDU step reduces to an average and can be computed efficiently as $\bm{\phi}_k=\sum_{j}^{p_s^k} \bm{y}^s_k\left[j\right] /p_s^k$, denoting by $\bm{y}^s_k\left[j\right]$ the $j^{\text{th}}$ subpart of $\bm{x}^s$ where a pattern is recognized (a non-zero activation in $\bm{z}_k^s$) and $p_s^k$ the number of non-zero activations in $\bm{z}_k^s$.
Furthermore, a closed formed solution is also available for the PerCDU step (\Cref{theorem:convergence}). 
These observations support the implementation of a federated learning version of this meta-algorithm, especially well-suited to the case of physiological data because of the aforementioned constraint (additional details can be found in \Cref{sec:federated}).

\vspace{-1ex}
\paragraph{Initialization.} 
Both the activations \(\bm{Z}\) and the personalization parameters $\bm{A}$ are initialized to zero.
This is common practice for \(\bm{Z}\) and empirical evidence suggest that this results in good performance for $A$.
In contrast, it is known that DL and CDL are sensitive to dictionary initialization~\citep{garcia2018convolutional,agarwal2014learning,ravishankar2020analysis}.
In the context of overcomplete DL, \citet{arora2013new} and \citet{agarwal2014learning} have provided upper bounds on the distance between the initial estimate and the true underlying dictionary to ensure complete recovery, as well as methods to achieve such theoretically good initializations.
Therefore, a dictionary initialization with a reasonable estimate of the expected patterns is recommended whenever possible.
In our experiments, the common atoms in $\Phi$ were initialized with a representative motif of the expected common structures (\textit{e.g.}, a complete gait cycle obtained from a healthy participant).

\vspace{-1ex}
\paragraph{Time Complexity.}
\label{sec:time_complexity}
Given that a basic step of CDL is an alternating minimization scheme (CDU+CSC), the complexity of $N_{\text{step}}$ of CDL is thus $\mathcal{C}[N_{\text{step}}](CDL)=N_{\text{step}}[\mathcal{C}(\text{CSC})+\mathcal{C}(\text{CDU})]$, with $\mathcal{C}(\text{CSC})=\mathcal{O}(SKLN(\log(N)+P))$ \citep{Charles24} and $\mathcal{C}(\text{CDU})=\mathcal{O}(SNKP)$ \citep{garcia2018convolutional}.
If the transformation $f$ is linear in $\bm{\phi}$, $\mathcal{C}(\text{PerCDU}) = \mathcal{C}(\text{CDU})$.
Therefore, the complexity of PerCDL can be written as $\mathcal{O}\left(\mathcal{C}[n_{\text{init}}+n_{\text{perso}}](CDL)+n_{\text{perso}}\mathcal{C}(\text{IPU})]\right)$. The extra cost of PerCDL compared to PopCDL is only the parameter update (IPU), which can be computed in parallel.
$\mathcal{C}(\text{IPU})$ is bounded by $\mathcal{O}(SKMNP T_{\text{grad}})$ with $T_{\text{grad}}$ the number of gradient steps.
However, \Cref{lemma:rotations} on orthogonal transformation (\Cref{sec:rotations_appendix}) proves that the IPU can be solved more efficiently if structural assumptions can be made on $f$.

%%%%%%%%%%%%%%%%%%%%%%%%%%%%%%%%%%%%%%%%%%%%%%%%%%%%%%%%%%%%

\subsection{Transformation Function} 
\label{sec:transformation_function}
The choice of the transformation function $f$ is paramount to the PerCDL framework and should be regarded as an integral aspect of downstream problem modelling.
Since our framework is general, any transformation function appropriate for the task at hand can be used.  
This section will provide a family of transformation functions that are well suited for, but not exclusive to, human physiology data.

\vspace{-1ex}
\paragraph{Time Warping Transformation.}
Let us consider a continuous signal $s:[0,1]\to \mathbb{R}^P$.
For $t \in [0, 1]$, the ``time-warped'' version of $s$ is $s_{\mathrm{warp}}\left(t\right) = s\left( \psi_{\bm{a}}\left(t\right) \right)$, where $\psi_{\bm{a}}:[0,1]\to[0,1]$ is a time warping function parameterized by $\bm{a}\in \mathbb{R}^M$.
In our context, a time-warp $\psi_{\bm{a}} : [0,1] \to [0,1]$ is applied to each atom $\bm{\phi} = \left( \phi_{l} \right)_{l \in [L]} \in (\mathbb{R}^P)^L$ (interpreted as a discrete-time signal) as, for any $i\in [L]$,
\begin{equation}
    \label{eq:transformation_warping}
    f\left( \bm{\phi}, \bm{a} \right)
    =
    \mathcal{T}_{\psi_{\bm{a}},\sigma}\left(\bm{\phi}\right),
    \quad
    [\mathcal{T}_{\psi_{\bm{a}},\sigma}\left(\bm{\phi}\right) ]_{i}
    = 
    \tilde{\bm{\phi}}\left(\psi_{\bm{a}}\left(t_i\right)\right),
\end{equation}
where \(t_i = i/L\) is the \(i\)th sampling time,
\(\tilde{\bm{\phi}} : [0, 1] \mapsto {\left(\mathbb{R}^P\right)}^L\) is an interpolation of \(\bm{\phi}\), and $\mathcal{T}_{\psi_{\bm{a}},\sigma}:(\mathbb{R}^P)^L\to(\mathbb{R}^P)^L$ is the composition of the interpolation of \(\bm{\phi}\) and the time warping \(\psi_{\bm{a}}\).

A smooth linear interpolation is used to interpolate \(\bm{\phi}\). 
Let $c_l(t)=(t-t_l)$ for any $(t,l)\in[0,1]\times [L]$. Then,
\begin{equation*}
    \tilde{\bm{\phi}}\left(t\right)
    = 
    {\textstyle\sum_{l=1}^{L}} w_l \left(\sigma, t\right)
    \left[ \phi_l + c_l(t) \left(\phi_{l+1}-\phi_l\right) \right],
\end{equation*}
where
\(
    w_l \left( \sigma, t \right) =
    \exp\left[- \frac{c_{l}(t)^2}{ 2\sigma^2} \right]
    /
    \sum_{l'=1}^{L} \exp\left[- \frac{c_{l'}(t))^2} {2\sigma^2} \right] ,
\) 
and $\sigma$ is a regularization parameter such that $w_l(\sigma,t)\to \mathds{1}_{ \left(l-\frac{1}{2}, l+\frac{1}{2}\right)}  \left( L \, t \right)$ as $\sigma$ goes to zero. 
Conveniently, the resulting transformation $f$ is linear in $\bm{\phi}$.  
Furthermore, \(w_l\) allows for a smooth, differentiable mapping $(\bm{\phi},\bm{a}) \mapsto f(\bm{\phi},\bm{a})$ so that gradient-based optimization can be leveraged for the IPU update.

\begin{figure}
  \centering
    \centering
    \includegraphics[width=0.4\textwidth]{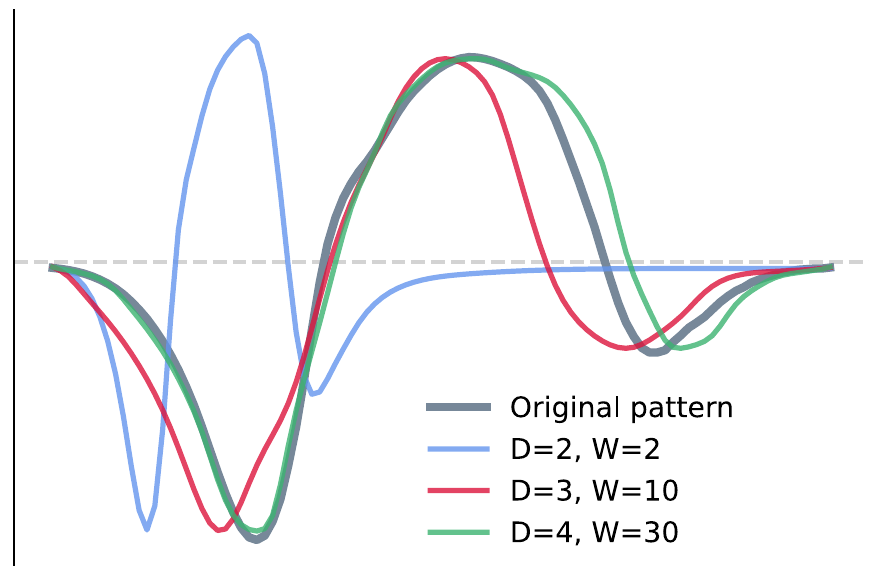}
    \caption{
        \textbf{Time Warping using Eqs. \eqref{eq:transformation_warping}-\eqref{eq:psi_def}}.
    }
    \label{fig:transfo_illustration}
\end{figure}

\vspace{-1ex}
\paragraph{Time Warping Function $\psi_{\bm{\alpha}}$.}
We opt for the time warping parameterization presented by \citet{celledoni2023deep}, which takes advantage of the Lie group structure of
\begin{align*}
    \mathcal{D}_\psi = \left\{g\in C^1([0,1]): g(0)=0,; g(1)=1,\; g'>0 \right\},
\end{align*}
the set of non-degenerate time warpings. (Degenerate time warpings allow $g'\geq 0$ such that $g$ has not necessary an inverse.) 
$\mathcal{D}_\psi$ is a group for the composition: the composition of two time warps is a valid time warp. 
Thus, an infinitesimal displacement around $g\in \mathcal{D}_\psi$ is not $g+\mathrm{d}g$, but $(\text{Id}+\mathrm{d}g)\circ g$, where $\mathrm{d}g$ is in $\mathsf{TM}_g=\{ g\in C^1([0,1]): g(0)=g(1)=0\}$, the tangent space at $g$~\citep{celledoni2023deep}. 
Given two hyper-parameters $D, W \in \mathbb{N}$ and $M=DW$, we define, for any $\bm{a}=(a^d_w)_{d, w\in[D] \times [W]} \in \Theta $, $t\in[0,1]$, and $d \in [D]$, $\psi_{\bm{a}}=\psi_{D,\bm{a}} \, \circ \, \ldots \, \circ \, \psi_{1,\bm{a}}$ as a composition of $D$ displacements:
\begin{equation}
	\label{eq:psi_def}
  \psi_{d,\bm{a}} = \text{Id} + \sum_{w=1}^{W} a^d_w \; b_w,\quad
  b_w(t) = \frac{\sin(w\pi t)}{w\pi}.
\end{equation}

In the following, we motivate the choice of $\Theta$ such that $\psi_{\bm{a}}$ is in the set of time warping $\mathcal{D}_\psi$, \textit{i.e.}, $\psi_{\bm{a}}[0] = 0$, $\psi_{\bm{a}}[1] = 1$ and $\psi_{\bm{a}}' > 0 $.
Note that $(b_w)_{w\geq 1}$ is a Fourier basis of $\mathsf{TM}_g$, thus by construction $\psi_{\bm{a}}[0] = 0$, $\psi_{\bm{a}}[1] = 1$. 
Moreover, $b_w$ is normalized such that $|b_w'|\leq 1$. 
As such, if for any $d\in[D]$, $\sum_{w=1}^{W} |a^d_w|< 1 $, then we have $\psi_{\bm{a},d}' \geq 0$ for any $d\in[D]$ and therefore by composition of time warping $\psi_{\bm{a}}\in \mathcal{D}_\psi$.
Hence, the parameter set $\Theta$ is chosen as
\begin{align*}
    \left\{\left(a^d_w\right)_{d, w \in [D] \times [W]}\in \R^{M} \;:\; {\textstyle\sum_{w=1}^{W}} |a^d_w| \leq 1,\;\; d \in [D] \right\} .
\end{align*}
The depth $D$ and width $W$ of the transformation both shape the extent of the variability around the common atoms. Tuning these hyper-parameters offers a trade-off between flexibility and interpretability of the transformation.
$D$ controls the Lipschitz constant of $\psi$, \textit{i.e.}, \textit{global dilation}, and $W$ controls the high frequency truncation, \textit{i.e.}, \textit{local dilation}.

\Cref{fig:transfo_illustration} shows the effect of $f(\cdot,\bm{a})$ defined with \eqref{eq:transformation_warping}-\eqref{eq:psi_def} for different choices of hyper-parameters $D,W$ and uniformly sampled $\bm{a}$ from $(-1,1)^M$ projected onto $\Theta$. 
Higher values of $D/W$ allow for greater flexibility, while lower values increase the granularity of the transformation. 
As illustrated in \Cref{fig:illustration}, this class of transformation is general enough to recover pathological gait cycles from a healthy one. 
Additional experiments on the sensitivity of the transformations with respect to $D$ and $W$ are presented in \Cref{sec:sensitivity_analysis} of the appendix.

%%%%%%%%%%%%%%%%%%%%%%%%%%%%%%%%%%%%%%%%%%%%%%%%%%%%%%%%%%%%

\subsection{Theoretical Guarantees: Toward Mixed-Effects Models}
\label{sec:theory}
PerCDL and mixed-effects models are naturally related since they both combine population-level and individual-level parameters.
In mixed-effects models, high-level structures in the population characteristics have been shown to boost the convergence rate of the MLE under mild assumptions, to a rate of $O(1/\sqrt{Sp})$ for population parameters \cite[Theorem 3]{nie2007convergence}, or $O(1/\sqrt{p})$ for individual parameters, with $S$ the number of individuals and $p$ the number of observations per individual.
This section will prove a similar convergence rate of $O(1/\sqrt{Sp})$ for the MLE related to PerCDL's common atoms, under reasonable assumptions.

\paragraph{Theoretical Setup.}
For the sake of clarity, our interest will be limited to a single common pattern in $\bm{\Phi}=\{\bm{\phi}_*\}$ (\textit{i.e.}, $K=1$), where $\bm{\phi}_* \in (\R^P)^L$ is the true common atom.
This implies that, for any $s\in[S]$, denoting the \(s^{\text{th}}\) parameter as $a_*^s\in \mathbb{R}^M$, the observation is generated as
{%
\setlength{\belowdisplayskip}{1.0ex} \setlength{\belowdisplayshortskip}{1.0ex}
\setlength{\abovedisplayskip}{1.0ex} \setlength{\abovedisplayshortskip}{1.0ex}
\begin{equation}
	\bm{x}^s= \bm{z}^s * f(\bm{\phi}_*,\bm{a}_*^s) + \bm{\epsilon}^s,\quad  \bm{\epsilon}^s \overset{\text{i.i.d.}}{\sim} \mathcal{N}(0_N,\sigma^2 I_N).
\end{equation}
}%

Let us denote by $\bm{y}^s\left[j\right]$ the $j^{\text{th}}$ subpart of $\bm{x}^s$ where a pattern is recognized by a non-zero activation in $\bm{z}^s$.
Now, provided that there is no overlap between atoms---a realistic assumption in physiological signals---and that $\bm{z^s}$ is known---another weak assumption since $l_0$ penalization is used and its sparse-coding solver is exact under non-overlapping condition \citep{Charles24}---for any $s, j \in [S] \times [p_s]$, $\bm{y}^s\left[j\right]$ can be decomposed as 
\begin{equation}
\label{eq:mixed_effect}
    \hspace{-.5em}
	\bm{y}^s\left[j\right] \!=\! f(\bm{\phi}_*,\bm{a}_*^s) \!+\! \bm{\epsilon}^s \left[j\right], 
 \;\; \bm{\epsilon}^s\left[j\right] \overset{\text{\scriptsize{i.i.d.}}}{\sim}\mathcal{N}(0_L,\sigma^2 I_L),
\end{equation}
where $p_s$ is the total number of observations and $\bm{a}^s_* \overset{\text{i.i.d.}}{\sim} \mathcal{P}$ the distribution of individual parameters. 

Even under these seemingly strong assumptions, the analysis will prove to be non-trivial because the assumptions on the distribution of individual parameters $\mathcal{P}$ are weak and the transformation function may be highly non-linear (TW).
This work provides the first theoretical analysis of DL with TW (\textit{cf.}, \citealp{yazdi2018time, xu2023generalized}). 
Identifying the minimal set of necessary assumptions on $\mathcal{P}$ is particularly important to gain insight into the group structure necessary for the recovery of common atoms.

%%%
\paragraph{Case 1: Global structure is absent.}
If there is no shared, global structure in the population (as assumed by IndCDL, \textit{i.e.}, $f(\bm{\phi}_*,\bm{a}_*^s)= \bm{a}_*^s \in (\mathbb{R}^P)^L$), then the MLE estimator is the local empirical mean $\frac{1}{p_s} \sum_{j=1}^{p_s} \bm{y}^s\left[j\right]$ and converges at a rate of $O\left( 1/\sqrt{p_s} \right)$.

%%%
\paragraph{Case 2: Global structure is present.}
In this case, we consider the following hypothesis to derive a convergence rate for the MLE on the common atom $\bm{\phi}_*$.
\begin{assumption}
	\label{assumption:structure}
    There exists a common atom $\bm{\phi}_* \in (\R^P)^L$ with individual parameters $\bm{a}^s_*\in \R^{M}$, and an operator $\mathcal{L}: \R^M \to \mathcal{M}_{L,L}(\R)$, such that for any $s\in[S]$, $\bm{\phi}^s_* = \mathcal{L} \left( \bm{a}^s_* \right) \,\bm{\phi}_*$.
    Moreover, $\mathcal{P}$ is a distribution on $\R^M$ such that $\mathbb{E}_{\mathcal{P}} ( | \mathcal{L}\left( \bm{a}^s_*\right)^\top \, \mathcal{L} \left( \bm{a}^s_*\right) | ) < \infty$ and $\mathbb{E}_{\mathcal{P}} ( \mathcal{L} \left( \bm{a}^s_* \right)^\top \mathcal{L} \left( \bm{a}^s_* \right) )$ is symmetric positive definite. 
    Its minimum eigenvalue is denoted by $\rho>0$. 
\end{assumption}
Under \Cref{assumption:structure}, the model is a linear mixed-effects model if and only if $\mathcal{P}$ is Gaussian \citep{nie2007convergence}. In \Cref{assumption:structure}, $\mathcal{P}$ needn't be Gaussian, \textit{e.g.}, it can be a multimodal distribution. 
Therefore, \Cref{assumption:structure} can be seen as a refinement of Condition 2 by \citet{nie2007convergence}.

\begin{lemma}
 \label{lemma:assumptions}
	The assumption \Cref{assumption:structure} is met in the two following cases:
	\vspace{-3ex}
	\begin{itemize}
	    \item Time warping transformations: For any $\bm{a} \in \Theta \subset \R^M$, $\mathcal{L}(\bm{a})=\mathcal{T}_{\psi_{\bm{a}},\sigma}$ where $\mathcal{T}_{\cdot,\sigma}$ and $\psi$ are defined in \eqref{eq:transformation_warping} and \eqref{eq:psi_def} with $\sigma$ small enough. $\mathcal{P}$ is a distribution supported on $\Theta$ such that for any $l\in[L]$, the $l$-th sample of the common atom $\bm{\phi}$ has a positive probability to be seen in a personalized atom $\bm{\phi}_*^s$. 
	    \item Rotations: For any $\bm{a} \in \Theta \subset \R^M$, $\mathcal{L}(\bm{a})\in \mathcal{O}_{P}(\R)$.
	\end{itemize}
	\vspace{-3ex}
\end{lemma}
More formal statements and proofs are postponed to \Cref{sec:fulltheory}.

The following additional hypothesis is needed to invoke a weighted law of large numbers in \Cref{theorem:convergence1} as it prevents an imbalance between the individuals' observations.
 \begin{assumption}
	\label{assumption:nb_observations}
	There exists a constant $C<+\infty$ such that for any $s\in[S]$, $1\leq p_s\leq C$. 
 \end{assumption}

Taking advantage of the fact that $\bm{\beta} \in(\R^P)^L \to f(\bm{a}, \bm{\beta})$ is linear, we derive the following theorem.
\begin{theorem}
\label{theorem:convergence1}
	Under \Cref{assumption:structure}, \Cref{assumption:nb_observations}, and assuming that the true parameters $(\bm{a}_s^*)_{s\in[S]}$ are given, for $S$ large enough, $\Sigma_{S,p}= \sum_{s=1}^S p_s \, \mathcal{L}\left( \bm{a}^s_* \right)^\top \, \mathcal{L}\left( \bm{a}^s_* \right)$ is invertible and the MLE estimator $\hat{\bm{\phi}}_{\mathrm{mle}} \sim \mathcal{N} \left( \bm{\phi}_*, (\Sigma_{S,p})^{-1} \right) $ is Gaussian. Moreover, $\hat{\bm{\phi}}_{\mathrm{mle}}$ converges toward $\bm{\phi}_*$ at a rate of $\mathcal{O} \big(1/\rho\sqrt{\sum_{s=1}^S p_s} \big)$ in probability. 
\end{theorem}

When $p_s=p$ for any $s\in[S]$, we obtain the same convergence rate as for population parameters in mixed-effect models $\mathcal{O} \left(1/\sqrt{Sp} \right)$ \cite[Theorem 3]{nie2007convergence}. 
Furthermore, the assumption that the true parameters $(\bm{a}_s^*)_{s\in[S]}$ are known can be relaxed by assuming that a $\sqrt{p}$-consistent estimator $(\bm{\hat{a}}_s^*)_{s\in[S]}$ of $(\bm{a}_s^*)_{s\in[S]}$ is available.
However, more technical conditions \cite[Condition 1]{nie2007convergence} would have to be involved and complicate the exposition.
Nevertheless, \Cref{theorem:convergence1} asserts that, under suitable conditions, PerCDL offers identical statistical guarantees as mixed-effects models for estimating the common parameters.

%%%%%%%%%%%%%%%%%%%%%%%%%%%%%%%%%%%%%%%%%%%%%%%%%%%%%%%%%%%%

\section{EXPERIMENTS}
\label{sec:experiments}
The meta-algorithm presented in \Cref{sec:meta-algorithm} was applied to synthetic (\Cref{sec:synthetic_data}) and real-world datasets (locomotion data in \Cref{sec:experiments:gait_analysis} and ECG data in \Cref{sec:app:ECG_analysis}). 
When applicable, PerCDL was compared with IndCDL and PopCDL.
Implementation details and additional numerical experiments can be found in \Cref{sec:numerical_details}. 

%%%%%%%%%% 

\subsection{Synthetic Data}
\label{sec:synthetic_data}
We generate a synthetic dataset with $S$ time series of equal length $N$. Each time series contains $r=3$ repetitions of $K=2$ common patterns of length $L$ (\Cref{fig:app:validation_common}).
Both patterns are personalized independently for each signal through TW (\Cref{eq:transformation_warping}), with parameters drawn uniformly random from $(-1, 1)^M$ and projected onto $\Theta$.
The signal-specific patterns are then placed at randomly chosen positions within each time series, with no overlap.
In all experiments, activations and personalization parameters are initialized to zero, while the dictionary is initialized with the personalized atom of the first signal.

%%%%
\paragraph{Identifiability.} 
We investigate the ability of IndCDL, PopCDL, and PerCDL to extract the common structures with an increase in the number of signals.
All methods are applied to multiple synthetic datasets that contains between \(S = 1\) and \(S = 1024\) time series of equal length \(N = 500\).
For IndCDL, the common structures are computed as the Euclidean barycenter of the individual atoms found in each signal.
While we also considered a barycenter computed with the Soft-DTW metric, this led to worse results.
Thus, we only consider the Euclidean barycenter to ensure a fair comparison against PopCDL and PerCDL.
The distance between the obtained common structures and the underlying ground truths are shown in \Cref{fig:conv}. 
(The shape of the common atoms estimated by the three methods is shown in \Cref{fig:supp:synth_conv} in the Appendix.)

\begin{figure}[!t]
\centering
    \includegraphics[width=.9\linewidth]{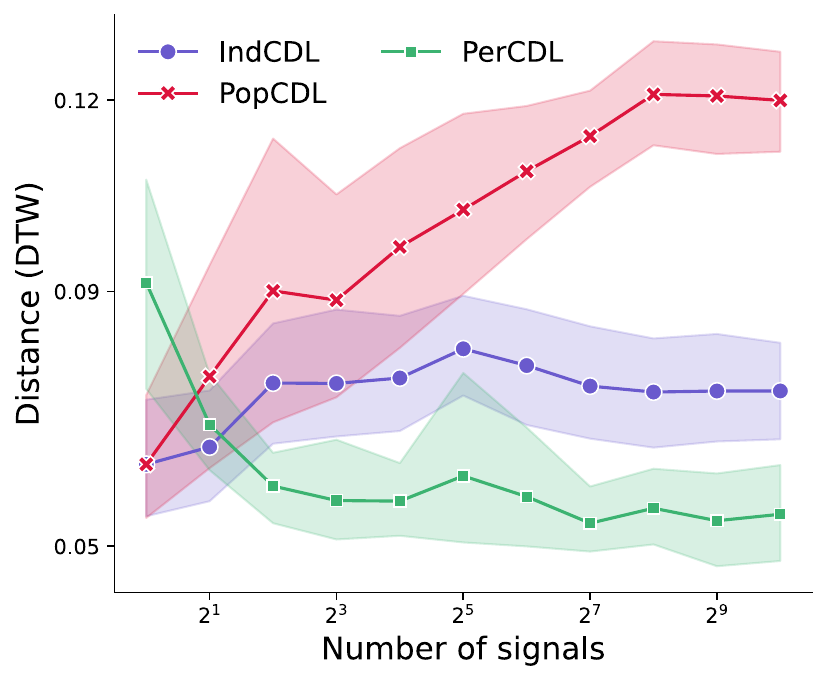}
\caption{
    \textbf{Convergence Toward the Common Structures.}
    Average distance between the common atoms identified in the synthetic experiment and the ground truth as a function of the dataset's size. 
    Shaded regions represent the $95 \%$ confidence interval.
}
\label{fig:conv}
\end{figure}

As the number of signals in the dataset increases, the quality of the common atoms identified by IndCDL and PopCDL degrades.
This suggests that the idiosyncrasies of each signal negatively affects estimation of the common structures.
Notably, IndCDL appears to perform better than PopCDL as the number of signals increases. 
We suspect that the common structures identified by PopCDL, by construction, do not match the local variations of the individual signals induced by the time warps, which could affect the segmentation performance of the CSC step. 
This segmentation error may in turn lead to CDU updates that drift further and further away from the true underlying common structures when the amount of individual atom variability increases. 
With IndCDL, segmentation is necessarily accurate since it uses the personalized structures of each signal directly.
In contrast with IndCDL and PopCDL, PerCDL ensures a good identifiability of the common shapes in the data that improves as the size of the dataset increases.

%%%%%%%%%%%%%%%%%%%%%%%%%%%%%%%%%%%%%%%%%%%%%%%%%%%%%%%%%%%%
\begin{figure}[!t]
\centering
    \includegraphics[width=.9\linewidth]{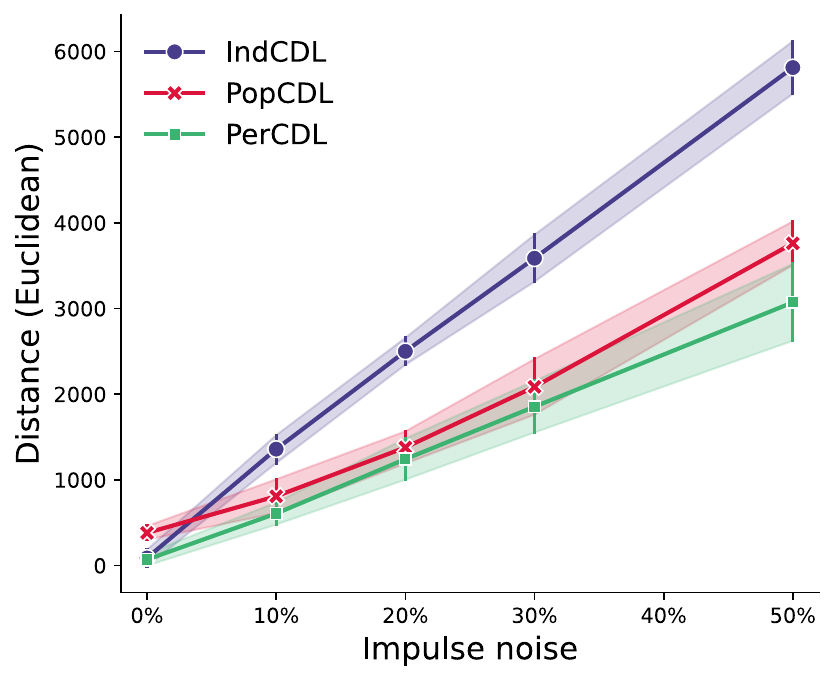}
\caption{
    \textbf{Reconstruction Error Under Impulse Noise Contamination}. Error bars represent two standard deviations.
}
\label{fig:noise}
\end{figure}

\paragraph{Robustness to Noise.} 
We now evaluate the reconstruction accuracy of PerCDL under impulse noise.
(Additional experiments with Gaussian noise can be found in  \Cref{sec:robustness_noise}). 
To that end, we generate a synthetic dataset of $S=32$ time series of equal length $N=1000$. 
In each time series, $r=3$ patterns of length $L=50$ are placed at random positions, with no overlap.
A quarter of the signals are then corrupted with impulse noise: $p$ percent of all samples---chosen at random---are polluted with either a positive or negative spike of uniformly random amplitude in the range $[2, 2.5)$.
The reconstruction errors as a function of the proportion of corrupted samples are presented in \Cref{fig:noise}.

While the performance of all methods degrades with an increase in noise, PerCDL proves more robust than IndCDL and PopCDL. 
Our results indicate that population-level structure information combined with the superior representation possibilities offered by the personalization step better withstand noise and artefacts in the input data.

%%%%%%%%%%%%%%%%%%%%%%%%%%%%%%%%%%%%%%%%%%%%%%%%%%%%%%%%%%%%

\begin{figure*}[!hbt]
\centering
  \begin{tabular}[b]{c}
    \includegraphics[width=.3\linewidth]{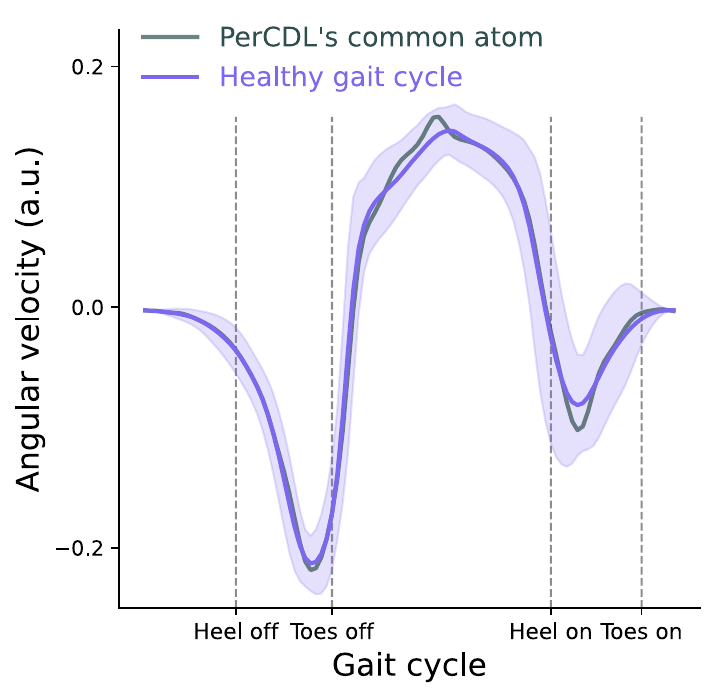} \\
    \small (a) PerCDL's common atom.
  \end{tabular} 
  \begin{tabular}[b]{c}
    \includegraphics[width=.56\linewidth]{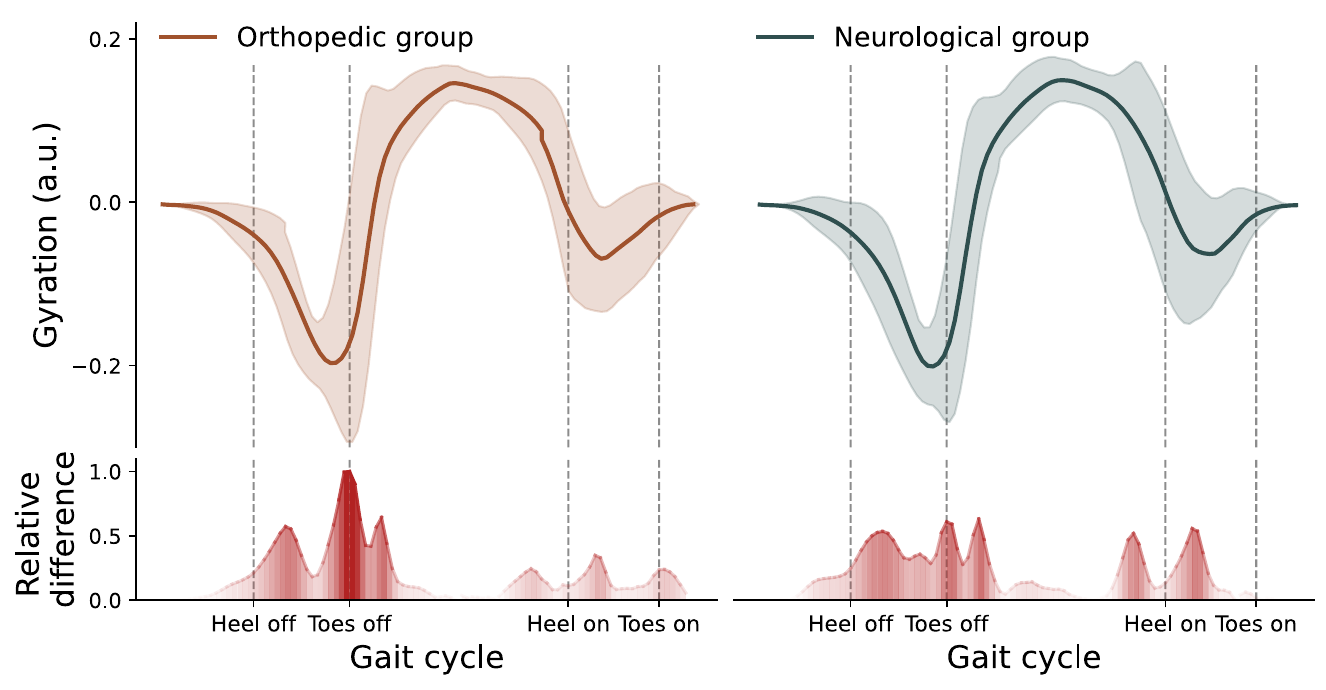} \\
    \small (b) Comparison between healthy and pathological participants.
  \end{tabular}
\caption{
    \textbf{Application of PerCDL to Locomotion Data}.
    (a) Common atom identified by PerCDL (grey) and averaged personalized gait cycle of the healthy population (blue). The shaded region represents two standard deviations.
    (b) Population-averaged personalized atoms learned by PerCDL in the Orthopedic (brown) and Neurological (green) groups, with two standard deviations (shaded areas). 
    The red-colored curves represent the normalized relative variability in the gait cycle for the current population compared with the Healthy group.
}
\label{fig:gait}
%\vspace{-3ex}
\end{figure*}

%%%%%%%%%%%%%%%%%%%%%%%%%%%
\subsection{Application to Gait Analysis}
\label{sec:experiments:gait_analysis}

We now apply PerCDL to real-world healthy and pathological locomotion data~\citep{truong2019} (more details are provided in \Cref{sec:app:locomotion}).
\(150\) participants were instructed to walk at their preferred speed in a straight line on a leveled surface for \(10\) meters, while an inertial measurement unit measured their foot angular velocity in the sagittal plane. 
Participants were divided into three equal-sized groups of size \(n=50\) (Healthy, Orthopedic, and Neurological) according to their health status.
The common dictionary was initialized with the gait cycle extracted from a randomly chosen healthy participant. 

One of the first task in gait analysis is to identify the gait cycles within each locomotion signal.
PerCDL's ability to successfully recover the gait cycles, through the activation matrix $\bm{Z}$, is summarized in \Cref{tab:gait_segmentation}.
We obtain results comparable to the state-of-the-art gait segmentation method introduced in \citet{voisard2023automatic}.
More details on this experiment can be found in \Cref{sec:app:gait_segmentation}.

\begin{table}[!hbt]
\centering
\caption{
    \textbf{Segmentation Performance of PerCDL vs. SOTA}. 
    FPR: False Positive Rate.
}
\begin{tabular}{l|l|l}
                                       & \textbf{Voisard et al.} & \textbf{PerCDL} \\ \hline
\textbf{Sensitivity}         & 0.959                   & 0.988           \\
\textbf{FPR} & 0.004                   & 0.033          
\end{tabular}
\label{tab:gait_segmentation}
\end{table}

However, PerCDL is not only capable to efficiently segment locomotion signals, it can also be used to represent, interpret and analyze such data. 
\Cref{fig:gait}(a) shows that the common atom identified (grey line) is similar to the average of the personalized atoms in the healthy population.
Intuitively, the shared characteristics within a population should roughly resemble the average of its healthy constituents.
Thus, the shared structure identified by PerCDL carries an intrinsic meaning and serves as a coherent foundation for signal-specific personalization through time warps.
This finding is promising as it could facilitate the monitoring of fine deviations from a ``healthy baseline'' that may be indicative of the start of a pathological process \citep{konig2016revealing, vidal2021perceptual}. 

We then examine the personalized atoms found by PerCDL in the different pathology groups (\Cref{fig:app:gait_personalization}). 
While all groups present similar gait cycles, discernible inter-individual differences can be identified in the foot kinematics during specific phases of the gait cycle. 
\Cref{fig:gait}(b) highlights this point via the representation of the normalized relative difference in local gait variability between each pathological group and the healthy population. 
Our results suggest that the Orthopedic group diverges from a healthy gait cycle mainly at the beginning of the swing phase ("Toes off"), while the neurological group differs mainly at the end of the stance phase (between "Heel off" and "Toes off") and right before the end of the swing phase ("Heel on"). 

More experiments, including a comparison between IndCDL, PopCDL and PerCDL, are available in the Appendix (\Cref{sec:app:locomotion}). 
Overall, our findings demonstrate the effectiveness of PerCDL for gait analysis in the identification of both common and personal structures. 

Furthermore, experiments on an ECG dataset in \Cref{sec:app:ECG_analysis} reveals that personalization parameters extract relevant characteristics from the input signals, such that they can be used as features in pathology classification.

%%%%%%%%%%%%%%%%%%%%%%%%%%%%%%%%%%%%%%%%%%%%%%%%%%%%%%%%%%%%

\section{DISCUSSION}

\subsection{Related Works}
\label{sec:related_works}

CDL~\citep{grosse2007shift, zheng2016efficient, morup2008shift} is a well-established method for time series representation that incorporates translation-invariance to the DL framework~\citep{kreutz2003dictionary,tovsic2011dictionary}.
Interestingly, CDL shares common roots with Convolutive Non-Negative Matrix Factorization (cNMF or cNMFsc, the latter imposing sparseness constraints; \citealp{smaragdis2006convolutive, o2008discovering, wang2009multiplicative}), which has been employed in speech separation and the analysis of articulatory movement primitives during speech production \citep{ramanarayanan2013spatio}. 
Specifically, \citet{ramanarayanan2013spatio, smaragdis2006convolutive} have analyzed individual speech patterns by examining their constituent atoms using a method similar to IndCDL.
Theoretical connections can also be made with the deep learning litterature: PopCDL may be understood as a very specific Vector Quantized Variational Auto-Encoder (VQ-VAE) model~\citep{van2017neural} if the common dictionary $\mathbf{\Phi} \in \mathbb{R}^{K \times L\times P}$ is interpreted as the discrete latent space. 
Nonetheless, the projection step onto the closest vector is performed in the latent space in VQ-VAE but occurs during the sparse-coding step in CDL.
% TODO Ajouter une ref de deep learning
Consequently, after the sparse coding step, we only need to learn the discrete embeddings (the atoms) from the observations. 
%With PerCDL, the latent space is not discrete because of the personalization parameters $\bm{a}$, which are signal-specific. 
%This makes PerCDL fundamentally different from VQ-VAE, contrary to PopCDL.

Multiple extensions of CDL were developed to further enforce structure into the atoms.
For instance, \citeauthor{soh2021group} proposed group-invariant DL, which complements the shift-invariance property with rotation invariance.
Others have introduced parametric forms on the atoms to analyze seismic data~\citep{chen2023parametric,turquais2018parabolic} or sketch vectorization~\citep{shaheen2017constrained}.
Parameterisation of the atoms is also central to PerCDL, though specifically to model global and local structures.

In parallel, the advent of individual-level data has precipitated the development of methodologies designed to create personalised representations in a plethora of applications from federated learning and unsupervised learning~\citep{dinh2020personalized,huang2023federated}.
In representation learning in particular, personalized variants of principle component analysis~\citep{ozkara2023personalized, shi2024personalized}, matrix/tensor decomposition~\citep{hu2023personalized, shi2023heterogeneous}, and dictionary learning~\citep{liang2024personalized} were introduced.
Although our work follows the aforementioned trend by proposing a personalized variant of CDL, it differs in that it considers \textit{personalizations} that are specific to time-series data. 
The most closely related approach is perhaps the Generalized Canonical Time Warping method (GCTW; \citealp{xu2023generalized}). 
GCTW was proposed as a means of resolving temporal misalignment between the input data and the atoms by jointly learning the dictionary and time warping matrices.
Unlike PerCDL, however, GCTW is a DL method that focuses on \textit{counteracting} temporal variability and therefore does not attempt to learn an interpretable representation of the local variations.

%%%%%%%%%%%%%%%%%%%%%%%%%%%

\subsection{Conclusion}
\label{sec:limitations}
This work introduced Personalized Convolutional Dictionary Learning (PerCDL), a framework for learning interpretable representations of time-series datasets that captures both global structures and local variations around these commonalities.
A meta-algorithm was proposed and its theoretical performance was evaluated.
Theoretical consistency guarantees were obtained under the typical conditions used in the mixed-effects models literature.
The performance of PerCDL was empirically demonstrated on synthetic and real-world human physiological signals.
% in $i)$ segmenting relevant patterns, $ii)$ effectively reconstructing the input signals, and $iii)$ extracting meaningful personalization parameters, 

The proposed framework is general in nature and could be extended to other time-series-related tasks by adapting the transformation function.
The performance of PerCDL on other modalities, (such as images; \citep{shaheen2017constrained}), is an interesting avenue for research and is left to future works.

The primary limitation of the proposed method is that of computational cost associated with the IPU step.
To alleviate any potential scalability issue, a federated-learning variant of PerCDL was proposed.
This represents a promising research direction that will be further refined in subsequent work.

%%%%%%%%%%%%%%%%%%%%%%%%%%%%%%%%%%%%%%%%%%%%%%%%%%%%%%%%%%%%
\subsubsection*{Acknowledgments}
K. Kim was supported by a grant from AWS AI to Penn Engineering's ASSET Center for Trustworthy AI.

%%%%%%%%%%%%%%%%%%%%%%%%%%%%%%%%%%%%%%%%%%%%%%%%%%%%%%%%%%%%
%\newpage
\bibliographystyle{unsrtnat}
\bibliography{references_paper}

%%%%%%%%%%%%%%%%%%%%%%%%%%%%%%%%%%
\clearpage
%%%%%%%%%%%%%%%%%%%%%%%%%%%%%%%%%%%%%%%%%%%%%%%%%%%%%%%%%%%%
\section*{Checklist}
 \begin{enumerate}

 \item For all models and algorithms presented, check if you include:
 \begin{enumerate}
   \item A clear description of the mathematical setting, assumptions, algorithm, and/or model. \textbf{Yes}. A clear description of the mathematical setting, assumptions and algorithm are presented throughout \Cref{sec:methodology}.
   \item An analysis of the properties and complexity (time, space, sample size) of any algorithm. \textbf{Yes}. The convergence rate of the method is presented in \Cref{sec:theory} and detailed in \Cref{sec:fulltheory}. The complexity of the algorithm (the federated learning version) is given in \Cref{sec:introduction} and detailed in \Cref{sec:federated}.
   \item (Optional) Anonymized source code, with specification of all dependencies, including external libraries. \textbf{Yes}. Source code for the methods presented will be made fully available online.
 \end{enumerate}

 \item For any theoretical claim, check if you include:
 \begin{enumerate}
   \item Statements of the full set of assumptions of all theoretical results. \textbf{Yes}. The no-overlap between atoms assumption is first presented at the end of \Cref{sec:introduction} and then at the beginning of \Cref{sec:theory}. The second assumption (known activations) is presented at the beginning of \Cref{sec:theory}. The final assumption (known personalization parameters) is given in \Cref{theorem:convergence1}.
   \item Complete proofs of all theoretical results. \textbf{Yes}. Informal proofs are provided in the main text (\Cref{sec:theory}), and complete proofs in the appendix (\Cref{sec:fulltheory}). 
   \item Clear explanations of any assumptions. \textbf{Yes}. The no-overlap between atoms assumption is discussed at the end of \Cref{sec:introduction}. The second assumption (known activations) is explained during its presentation (\Cref{sec:theory}). The final assumption (known personalization parameters) is discussed below \Cref{theorem:convergence1}.
 \end{enumerate}

 \item For all figures and tables that present empirical results, check if you include:
 \begin{enumerate}
   \item The code, data, and instructions needed to reproduce the main experimental results (either in the supplemental material or as a URL). \textbf{Yes}. All details needed to reproduce the numerical experiments are given in \Cref{sec:numerical_details}, with links to the datasets. Furthermore, code for the methods introduced in this work will be made available as a URL. 
   \item All the training details (e.g., data splits, hyperparameters, how they were chosen). \textbf{Yes}. All details regarding the numerical experiments can be found in the main text (\Cref{sec:experiments}) and in \Cref{sec:numerical_details}. 
         \item A clear definition of the specific measure or statistics and error bars (e.g., with respect to the random seed after running experiments multiple times). \textbf{Yes}. Information on the measure and error bars is given alongside each figure. 
         \item A description of the computing infrastructure used. (e.g., type of GPUs, internal cluster, or cloud provider). \textbf{Yes}. Information on the computing infrastructure used is given in \Cref{sec:numerical_details}.
 \end{enumerate}

 \item If you are using existing assets (e.g., code, data, models) or curating/releasing new assets, check if you include:
 \begin{enumerate}
   \item Citations of the creator If your work uses existing assets. \textbf{Yes}. The authors of the CSC optimization scheme \citep{Charles24} and the gait segmentation method \citep{voisard2023automatic} were cited.
   \item The license information of the assets, if applicable. \textbf{Not Applicable}.
   \item New assets either in the supplemental material or as a URL, if applicable. \textbf{Yes}. The authors code for the methods discussed in the paper will be available via a URL.
   \item Information about consent from data providers/curators. \textbf{Not Applicable}.
   \item Discussion of sensible content if applicable, e.g., personally identifiable information or offensive content. \textbf{Not Applicable}.
 \end{enumerate}

 \item If you used crowdsourcing or conducted research with human subjects, check if you include:
 \begin{enumerate}
   \item The full text of instructions given to participants and screenshots. \textbf{Not Applicable}.
   \item Descriptions of potential participant risks, with links to Institutional Review Board (IRB) approvals if applicable. \textbf{Not Applicable}.
   \item The estimated hourly wage paid to participants and the total amount spent on participant compensation. \textbf{Not Applicable}.
 \end{enumerate}

 \end{enumerate}

%%%%%%%%%%%%%%%%%%%%%%%%%%%%%%%%%%%%%%%%%%%%%%%%%%%%%%%%%%%%
\appendix
\newpage
\onecolumn
\section{COMPLETE PRESENTATION OF THEORETICAL GUARANTEES}
\label{sec:fulltheory}

\subsection{Rotations}
\label{sec:rotations_appendix}
\Cref{sec:time_complexity} mentioned a complexity of the IPU step in $\mathcal{O}(SKMNP T_{\text{grad}})$ for the general case, with $T_{\text{grad}}$ the number of gradient steps.
However, this complexity can be decreased if more structure is assumed on $f$.

Let us consider the particular case of orthogonal transformations.
Some multivariate time series, for example hand-writing-related signals~\citep{vayer2022time}, naturally exhibit rotations in space.
To model such data in the PerCDL framework, let us define the rotation transformation of $\bm{\phi}$, $f_r(\bm{\phi},\bm{a})=\bm{a}\bm{\phi}$, for any $\bm{\phi},\bm{a} \in (\mathbb{R}^P)^L\times \mathcal{O}_P$.
Let $\bm{x}^s$ be a signal whose segments with non-zero activation are written $X_i$.
\Cref{lemma:rotations} shows that the IPU update can be solved directly.

\begin{lemma}
\label{lemma:rotations}
    Given an atom $\bm{\phi} \in (\mathbb{R}^P)$ and raw data $X_i\in (\mathbb{R}^P)$ for any $i\in[Q]$, where $Q$ is the number of repeated noisy patterns in an individual time series, the solution of $\argmin{O\in \mathcal{O}_P} \sum_{i=1}^m |X_i-O\phi|^2$ is $V U^\top$ where $\bm{\phi} \sum_{i=1}^m X_i^\top/m = U \Sigma V^\top$ is a singular values decomposition.
\end{lemma}

\begin{proof}
    We have, 
    \begin{align}
        \sum_{i=1}^m |X_i-O\phi|^2 &= \sum_{i=1}^m |X_i|^2+|O\phi|^2-2\operatorname{tr}(X_i^\top O\phi) \\
        &= \left[\sum_{i=1}^m |X_i|^2+|\phi|^2 \right] -2\operatorname{tr}\left(\phi\sum_{i=1}^m X_i^\top O\right),
    \end{align}
    where we used that $|O\phi|^2=|\phi|^2$ since $O\in\mathcal{O}_P$. 
    Then, 
    \begin{equation}
        \argmax{O\in \mathcal{O}_P}\operatorname{tr}\left(\phi\sum_{i=1}^m X_i^\top O\right) = \argmin{O\in \mathcal{O}_P} \sum_{i=1}^m |X_i-O\phi|^2,
    \end{equation}
    and 
    \begin{equation}
        \operatorname{tr} \left( \phi\sum_{i=1}^m X_i^\top O \right) = \operatorname{tr} \left( \Sigma V^\top O U \right),
    \end{equation}
    which gives $O=VU^\top$ since $\Sigma$ is a non-negative diagonal and $O$ is orthogonal.
\end{proof}

Given \Cref{lemma:rotations}, the complexity of IPU drops from $\mathcal{O}(SKMNPT_{\mathrm{grad}})$ to $\mathcal{O}(SKP^3)$. 
A similar update rule was shown in \cite{soh2021group} for group-invariant DL.

%%%%%%%%%%%%%%%%%%%%%%%%%%%%%%%%%%%%%%%%%%%%%%%%%%%%
\subsection{Toward Mixed-Effects Models}

%%%%%%%%%%%%%%%%%%%%%%%
\subsubsection{Alternative Formulation of \Cref{lemma:assumptions} and Proof}
\Cref{sec:theory} introduced \Cref{assumption:structure}, the foundation hypothesis used to derive a convergence rate for the MLE on the common atom $\bm{\phi}_*$.
\Cref{lemma:assumptions} then proved that \Cref{assumption:structure} is met most of the time in practice.

In what follows, we propose an alternative, more formal statement of this Lemma.
\begin{lemma}
 \label{lemma:app:assumptions}
	Assumption \Cref{assumption:structure} is met in the two following cases:
	\begin{itemize}
		\item (Time warping transformations) For any $\bm{a} \in \Theta \subset \R^M$, $\mathcal{L}(\bm{a})=\mathcal{T}_{\psi_{\bm{a}},\sigma}$ where $\mathcal{T}_{\cdot,\sigma}$ and $\psi$ are defined in Equations \eqref{eq:transformation_warping} and \eqref{eq:psi_def} with $\sigma$ small enough.
		For any $s\in[S]$, $\bm{a}^s_* \overset{\text{i.i.d.}}{\sim} \mathcal{P}$ where $\mathcal{P}$ is a distribution supported on $\Theta$ such that for any $l\in[L]$,
        {%
        \setlength{\belowdisplayskip}{0.0ex} \setlength{\belowdisplayshortskip}{0.0ex}
        \setlength{\abovedisplayskip}{0.5ex} \setlength{\abovedisplayshortskip}{0.5ex}
		\begin{equation}\label{eq:cond_S} 
		S_l=\sum_{i=1}^L \mathbb{P}( L\psi(a)(i/L)\in (l-1/2,l+1/2) )>0 .
		\end{equation}
        }%
		 \item (Rotations) For any $\bm{a} \in \Theta \subset \R^M$, $\mathcal{L}(\bm{a})\in \mathcal{O}_{P}(\R)$.
	\end{itemize}
\end{lemma}

\begin{proof}
	The case of rotations is trivial since for any $s \in [S]$, $\mathcal{L} \left( \bm{a}^s_*\right)^\top \, \mathcal{L} \left( \bm{a}^s_* \right) = I_L$.

	The case of discrete temporal reparameterizations is less obvious. 
	 For any $\bm{x} \in (\R^P)^L$, we have,
	 \begin{equation}
		\left( \mathbb{E}_{\bm{a} \sim \mathcal{P}} \left( \mathcal{T}_{\psi_{\bm{a}},\sigma}^\top \, \mathcal{T}_{\psi_{\bm{a}},\sigma} \right) x \right)^\top x
        = \mathbb{E}\left((\mathcal{T}_{\psi_{\bm{a}},\sigma}^\top \, \mathcal{T}_{\psi_{\bm{a}},\sigma} \, x)^\top x\right) 
		= \mathbb{E}\left((\mathcal{T}_{\psi_{\bm{a}},\sigma} \, x)^\top \, \mathcal{T}_{\psi_{\bm{a}},\sigma} \, x\right).
	 \end{equation}
  
	 We have $\mathbb{E}\left((\mathcal{T}_{\psi_{\bm{a}},\sigma} \,x)^\top \, \mathcal{T}_{\psi_{\bm{a}},\sigma} \, x\right) = \sum_{i=1}^L \mathbb{E}\left([\mathcal{T}_{\psi_{\bm{a}},\sigma} \, x]_i^2 \right)$. 
	 For any $i\in[L]$, denoting by $U_l=x_l +(L\psi_{\bm{a}}(i/L)-l)(x_{l+1}-x_l)$,
	 \begin{align}
		\mathbb{E}\left([\mathcal{T}_{\psi_{\bm{a}},\sigma}x]_i^2\right)
		&=\mathbb{E}\left(\left(\sum_{l=1}^{L} w_l(\psi_{\bm{a}},\sigma,i)U_l\right)^2\right) \\
		&= \mathbb{E}\left(\sum_{l=1}^{L}\sum_{m=1}^{L} w_l(\psi_{\bm{a}},\sigma,i)w_m(\psi_{\bm{a}},\sigma,i)U_l U_m\right)\\
		&=\sum_{l=1}^{L}\sum_{m=1}^{L} \underbrace{\mathbb{E}\left(w_l(\psi_{\bm{a}},\sigma,i)w_m(\psi_{\bm{a}},\sigma,i) U_l U_m \right)}_{=\Delta_{k,m,i}}. 
	 \end{align}
	 
	 With $C_{l} = \left( l-\frac{1}{2}, l+\frac{1}{2} \right)$ for any $l\in[L]$, note that $w_l\left(\psi(\bm{a}\right),\sigma, i)\approx \mathds{1}_{ C_{l}} \left(L\psi_{\bm{a}}(i/L) \right)$ for any $i\in[L]$.
	 Therefore, computing $\Delta^*_{l,m,i} = \mathbb{E} \left( \mathds{1}_{C_{l}} \left( L \, \psi_{\bm{a}} \, (i/L) \right) \mathds{1}_{C_{m}} \left( L \, \psi_{\bm{a}} \, (i/L) \right) U_l U_m \right)$ for any $(l,m,i)\in[L]^3$ is interesting.
	 The proof of $\Delta^*_{l,m,i}\approx \Delta_{l,m,i}$ is postponed to the end of the current proof.
	  
	 We have,
	 \begin{align}
		\Delta^*_{l,m,i}
        &= \mathbb{E}\left( \mathds{1}_l(m)\mathds{1}_{C_{l}} \left( L \, \psi_{\bm{a}} \, (i/L) \right) U_l^2 \right) \\
		&= \mathds{1}_l(m)\mathbb{E}\left( U_l^2 \,|  L \, \psi_{\bm{a}} \, (i/L) \in C_{l}  \right) \mathbb{P} \left( L \, \psi_{\bm{a}} \, (i/L) \in C_{l}\right).
	 \end{align}
	 
	 Then, denoting by $P_{i,l}= \mathbb{P} \left( L \, \psi_{\bm{a}} \, (i/L) \in C_{l} \right)$, $S_l=\sum_{i=1}^L P_{i,l}>0$ (by \eqref{eq:cond_S}) and $\bar{P}_{i,l}=P_{i,l}/S_l$ for any $(i,l)\in[L]^2$,
	 \begin{align}
		\sum_{i=1}^L \sum_{l=1}^{L} \sum_{m=1}^{L}\Delta^*_{l,m,i}
		&= \sum_{i=1}^L \sum_{l=1}^{L} \mathbb{E} \left( U_l^2 \,| L \, \psi_{\bm{a}} \, (i/L) \in C_{l} \right) P_{i,l} \\
		& \geq \sum_{l=1}^{L}S_l \sum_{i=1}^L \left[ x_l \left( 1-\alpha_{i,l} \right)+x_{l+1} \alpha_{i,l} \right]^2 \frac{P_{i,l}}{S_l},
	 \end{align}
	 where we applied the Jensen inequality and wrote $\alpha_{i,l} = \mathbb{E} \left( L \, \psi_{\bm{a}} \, (i/L)-l \; | \; L\psi_{\bm{a}}(i/L)\in C_{l} \right)$ for any $(i,l)\in[L]^2$. 
	 Note that $\alpha_{i,l}\in (-1/2,1/2)$ since the expectation is taken on the event $L \, \psi_{\bm{a}} \, (i/L) \in C_{l}$.
	 
     \begin{align}
    	\sum_{i=1}^L \sum_{l=1}^{L}\sum_{m=1}^{L}\Delta^*_{l,m,i}
    	& \geq
    	\sum_{l=1}^{L}S_l \sum_{i=1}^L \left( x_l(1-\alpha_{i,l})+x_{l+1}\alpha_{i,l} \right)^2 \bar{P}_{i,l} \\
    	&\geq \sum_{l=1}^{L}S_l \left( x_l(1-\sum_{i=1}^L \bar{P}_{i,l} \alpha_{i,l})+x_{l+1}\sum_{i=1}^L \bar{P}_{i,l} \alpha_{i,l} \right)^2 
    	= x^\top \bar{D}^\top  \mathbf{S} \bar{D} x,
     \end{align}
	 where 
	 \begin{align}
	    &B_1 = \left( 1-\sum_{i=1}^L \bar{P}_{i,1} \alpha_{i,1} \right), \quad 
	    \bar{D}= \begin{pmatrix} B_1 &\sum_{i=1}^L \bar{P}_{i,1} \alpha_{i,1} & 0 & (0) \\
	    	0 & \left(1-\sum_{i=1}^L \bar{P}_{i,2} \alpha_{i,2}\right)&\sum_{i=1}^L \bar{P}_{i,2}\alpha_{i,2} & (0) \\
	        0 & 0 & \ldots & \ldots \end{pmatrix} \\
	    & \mathbf{S} = \operatorname*{diag}\left( (S_l)_{l\in[L]} \right).
	 \end{align}
	$\mathbf{S}$ is a positive diagonal matrix by \eqref{eq:cond_S} and $\bar{D}$ is a bi-diagonal matrix with a positive diagonal, thus $x^\top \bar{D}^\top  \mathbf{S} \bar{D} x>0$ is positive.

	Then, coming back to $\Delta^*_{l,m,i}\approx \Delta_{l,m,i}$, for any $x\in(\R^P)^L$,
	 \begin{equation}
		\sum_{i =1}^{L}\sum_{l=1}^{L}\sum_{m=1}^{L} \Delta_{k,m,i}= x^\top \tilde{W}_\sigma x
	\end{equation}
	where $\tilde{W}_\sigma$ is a matrix with coordinates which are functions of element of the form $\mathbb{E}\left( w_l \left( \psi_{\bm{a}}, \sigma, i \right) \, w_m \left(\psi_{\bm{a}}, \sigma, i \right) \, E_{i,l,m} \right) $ with $|E_{i,l,m}|\leq L^2$.
	Using that $w_l\left(\psi_{\bm{a}},\sigma,i\right) \in [0,1]$ converges almost surely toward $\mathds{1}_{ C_{l}}\left( L \, \psi_{\bm{a}} \, (i/L) \right)$ as $\sigma\to 0^+$, we have, using the Dominated Convergence Theorem, that $\tilde{W}_\sigma$ converges toward a matrix $\tilde{W}_*$ such that:
	\begin{equation}
		x^\top \tilde{W}_* x=\sum_{i=1}^L \sum_{l=1}^{L}\sum_{m=1}^{L}\Delta^*_{l,m,i}\geq x^\top \bar{D}^\top  \mathbf{S} \bar{D} x>0
	\end{equation}
	Therefore, since the space of invertible matrix is open, there exists $\sigma_0>0$ such that for any $\sigma>\sigma_0$, we have that $\mathbb{E}_{A\sim \mathcal{P}}(\mathcal{T}_{\psi_{\bm{a}},\sigma}^\top\mathcal{T}_{\psi_{\bm{a}},\sigma})$ is a symmetric positive definite matrix.
\end{proof}

%%%%%%%%%%%%%%%%%%%%%%%
\subsubsection{Circling Back to \Cref{lemma:assumptions}}
Let us define $\mathsf{T_\psi}^l = \left(\exists i \in [L]: L \, \psi_{\bm{a}} \, \left(i/L\right) \in \left(l-1/2,l+1/2\right) \right)$ for any $l\in[L]$.
For any $\epsilon>0$,
\begin{equation}
    P_{Id}^\epsilon
    = \mathbb{P}(\cap_{i\in[L]} \left( L \, \psi_{\bm{a}} \, \left(i/L \right)\in (i-\epsilon,i+\epsilon) \right)) 
    \leq \mathbb{P}(\cap_{l=1}^L\mathsf{T_\psi}^l)^L  \leq \prod_{l=1}^L \mathbb{P}(\mathsf{T_\psi}^l)\leq\prod_{l=1}^L S_l,
\end{equation}
and $P_{Id}^\epsilon>0$ as long as $\mathcal{P}$ has enough mass around zero since $\left(\bm{a}=0\right) = \left(\psi_{\bm{a}} = \text{Id} \right)$.  

Therefore, condition \eqref{eq:cond_S} from \Cref{lemma:app:assumptions} can be considered as very weak.
Informally, on the event $\left(L\psi_{\bm{a}}(i/L)\in (l-1/2,l+1/2)\right)$, we have $[\mathcal{T}_{\psi_{\bm{a}},\sigma}(\bm{\phi})]_i\approx \phi_l+(L\psi_{\bm{a}}(i/L)-l)(\phi_{l+1}-\phi_l)\approx \phi_l$ as $\sigma \ll 1$.

Thus, equation \eqref{eq:cond_S} states that, for any $l\in[L]$, the $l$-th sample of the common atom $\bm{\phi}$ has a positive probability to be seen in a personalized atom $\bm{\phi}_*^s$.
This statement is the one presented in \Cref{lemma:assumptions} from the main text.

%%%%%%%%%%%%%%%%%%%%%%%
\subsubsection{Proof of \Cref{theorem:convergence1}}
For completeness sake, this section will reason on a slightly more complete version of \Cref{theorem:convergence1}.
The following exposition naturally directly transposes to \Cref{theorem:convergence1}.

\begin{theorem}
\label{theorem:convergence}
	Under \Cref{assumption:structure}, \Cref{assumption:nb_observations}, and assuming that the true parameters $(\bm{a}_s^*)_{s\in[S]}$ are given, there exists almost surely an integer $\tilde{S}$, such that for any $S\geq \tilde{S}$, $\Sigma_{S,p}= \sum_{s=1}^S p_s \, \mathcal{L}\left( \bm{a}^s_* \right)^\top \, \mathcal{L}\left( \bm{a}^s_* \right)$ is invertible. 
	Then, for any $S\geq \tilde{S}$, the MLE of the common atom $\bm{\phi}_*$ is
	\begin{equation}
		\label{eq:MLE_Phi}
		\hat{\bm{\phi}}_{mle} = (\Sigma_{S,p})^{-1} Y_\Sigma, \quad Y_\Sigma=\sum_{s=1}^S \sum_{j=1}^{p_s}\mathcal{L} \left( \bm{a}^s_* \right)^\top[j] y^s[j],
	\end{equation}
	and thus we have $\hat{\bm{\phi}}_{mle} \sim \mathcal{N} \left( \bm{\phi}_*, (\Sigma_{S,p})^{-1} \right)$. 
	Moreover, $\hat{\bm{\phi}}_{mle}$ converges toward $\bm{\phi}_*$ at a rate of $\mathcal{O} \left(1/\rho\sqrt{\sum_{s=1}^S p_s}\right)$ in probability. 
\end{theorem}

\begin{proof}
    First, by \cite[Theorem 3]{jamison1965convergence}, which is an extension of the Strong Law of Large Numbers (SLLN) to the weighted averages, we have that almost surely $\Sigma_{S,p}/\sum_{s=1}^S p_s\to \mathbb{E}_{\mathcal{P}} \left( \mathcal{L} \left( \bm{a}^s_* \right)^\top \, \mathcal{L} \left( \bm{a}^s_* \right) \right)$ as $S$ goes to infinity. 
    Indeed, \cite[Theorem 3]{jamison1965convergence} states that SLLN applies if and only if $\limsup_{x\to\infty} \mathsf{S}(x)/x <\infty $ with $\mathsf{S}(x)=\#\{n\geq 0: \sum_{s=1}^n p_s/p_n \leq x \} $, $\left( \bm{a}^s_* \right)_{s \in [S]}$ i.i.d. and $\mathbb{E}_{\mathcal{P}} \left( \left| \mathcal{L} \left( \bm{a}^s_*\right)^\top \, \mathcal{L} \left( \bm{a}^s_* \right) \right| \right) < \infty$.
    This condition is met since for any $n\in[S]$, $\mathsf{S}(\sum_{s=1}^n p_s/(C+1))\leq n$ because by \Cref{assumption:nb_observations}, for any $n\in[N]$,
     \begin{equation}
	\sum_{s=1}^n p_s/C\leq \sum_{s=1}^n p_s/p_n\leq \sum_{s=1}^n p_s.
    \end{equation} 
    Thus, for any $n\in[S]$, $(C+1)\mathsf{S}(\sum_{s=1}^n p_s/(C+1))/\sum_{s=1}^n p_s \leq C+1$.
    This ensures $\limsup_{x\to\infty} \mathsf{S}(x)/x <\infty $ since $x\to \mathsf{S}(x)$ is not decreasing.

    Moreover, the set of invertible matrices of $\mathcal{M}_{L,L}(\R)$ is open. Therefore, there exists almost surely $\tilde{S}$ such that, for any $S\geq \tilde{S}$, $\Sigma_{S,p}/\sum_{s=1}^S p_s $ is invertible since $\mathbb{E}_{\mathcal{P}} \left( \mathcal{L}\left( \bm{a}^s_* \right)^\top \, \mathcal{L}\left( \bm{a}^s_* \right) \right)$ is invertible (\Cref{assumption:structure}).
    In addition, $\Sigma_{S,p}$ is symmetric positive as a sum of symmetric positive matrices. Therefore, for any $S\geq \tilde{S}$, $\Sigma_{S,p}$ is symmetric definite positive and the minimum eigen values of $\Sigma_{S,p}/\sum_{s=1}^S p_s$ converges to $\rho$ almost surely.

    Secondly, the log likelihood related to $\bm{\phi}_*$ in \eqref{eq:mixed_effect} is 
    \begin{equation}
        F: \bm{\phi}_* \in (\R^P)^L \mapsto  \sum_{s=1}^S \sum_{j=1}^{p_s} \left| y^s[j] - \mathcal{L} \left( \bm{a}^s_* \right) \, \bm{\phi}_* \right|^2 / \left( 2\sigma^2 \right) + L \, \log\left( 2\pi\sigma^2 \right) /2.
    \end{equation}
    Then, writing $\nabla F(\hat{\Phi})=0$ yields,
        \begin{equation}
	       \Sigma_{S,p}\hat{\bm{\phi}}_{mle} = Y_\Sigma.
        \end{equation}
    For any $S\geq \tilde{S}$, $\Sigma_{S,p}$ is invertible.
    Therefore, we have \eqref{eq:MLE_Phi}.
    
    Hence, 
    \begin{align}
	   \hat{\bm{\phi}}_{mle} - \bm{\phi}_*
    &= (\Sigma_{S,p})^{-1} Y_\Sigma- \bm{\phi}_* \\
    &= (\Sigma_{S,p})^{-1} \left[\sum_{s=1}^S p_s \, \mathcal{L} \left( \bm{a}^s_* \right)^\top \, \mathcal{L} \left( \bm{a}^s_* \right) \, \bm{\phi}_* \right] - \bm{\phi}_* + \left( \Sigma_{S,p} \right)^{-1} \sum_{s=1}^S \sum_{j=1}^{p_s} \, \mathcal{L} \left( \bm{a}^s_* \right)^\top[j] \epsilon^s[j] \\
    &= \sum_{s=1}^S \sum_{j=1}^{p_s} \left( \Sigma_{S,p} \right)^{-1} \, \mathcal{L}\left( \bm{a}^s_* \right)^\top[j] \epsilon^s[j].
    \end{align}
    
    We sum independent normal variables of covariance matrix $\sigma^2 \left( \Sigma_{S,p} \right)^{-1} \mathcal{L} \left( \bm{a}^s_* \right)^\top \left[ \left(\Sigma_{S,p} \right)^{-1} \mathcal{L} \left( \bm{a}^s_* \right) \right]^\top$, which yields,
    \begin{equation}
	   \hat{\bm{\phi}}_{mle} - \bm{\phi}_* \sim \mathcal{N} \left( 0, \left( \Sigma_{S,p} \right)^{-1} \right).
    \end{equation}
    
    Asymptotically, the largest eigen value of $\left( \Sigma_{S,p} \right)^{-1}$ is $1/\left( \rho\sum_{s=1}^S p_s \right)$, resulting in a convergence rate of $1 / \left(\rho\sum_{s=1}^S p_s\right)$ in probability.
\end{proof}

%%%%%%%%%%%%%%%%%%%%%%%%%%%%%%%%%%%%%%%%%%%%%%%%%%%%%%%%%%%%
\section{FEDERATED LEARNING META-ALGORITHM}
\label{sec:federated}

The expression of the MLE from \Cref{theorem:convergence} was used to adapt the meta-algorithm \cref{alg:PerCDL} to a Federated Learning framework under the non-overlapping assumption
In \Cref{alg:Federated_learning}, the Global Server aggregates information from the Local Servers through a weighted average. 
Each communication, from the Local Servers to the Central Server or vice versa, has a memory footprint of $\mathcal{O}(L)$ since the the transmitted vectors have nearly the same shape as the common atom. 
At each step, the complexity on the Local Server is $\mathcal{C}(\text{IPU})+\mathcal{C}(\text{CSC})=\mathcal{O}(max(KMT_{grad}N,KLN(\log(N)+P))$. 
In our experiments, $M = DW \approx 50 < L=100$ and $P<20$. 
Moreover, $n_{init}=n_{perso}=5$ is enough to reach convergence. 
Hence a final complexity per Local Server of $\mathcal{O}(N\log(N))$.

\IncMargin{2em}
\SetKwComment{Comment}{// }{}
\SetKw{Input}{Inputs}
\SetKw{Output}{Outputs}

\begin{algorithm*}[ht]
\caption{A Federated Learning algorithm to find an approximated solution of \eqref{eq:PerCDL}. Remark that $\hat{\bm{\Phi}}$ is the personalized dictionary defined in \eqref{eq:personalized_dictionnary} and $\bm{\Phi}$ is the common dictionary. We denote by $\bm{y}^s_k\left[j\right]$ the $j^{\text{th}}$ subpart of $\bm{x}^s$ where a pattern is recognized by a non-zero activation in $\bm{z}_k^s$ and $p_s^k$ the number of non-zero activation in $\bm{z}_k^s$. This algorithm works under the assumptions of \Cref{sec:theory}. The notations of \Cref{assumption:structure} are used.}
\label{alg:Federated_learning}
\Input{$\bm{x^s} \in \R^{S \times N\times P}$, $f:(\mathbb{R}^P)^L\times \mathbb{R}^M \mapsto (\mathbb{R}^P)^L $} \\ \Output{$\bm{\Phi} \in \R^{K \times L\times P}$, $\bm{A} \in \R^{S \times K \times M}$, $\bm{Z} \in \R^{S \times K \times N-L+1}$}
\BlankLine

\setcounter{AlgoLine}{0}

$\bm{\Phi}, \bm{A}, \bm{Z} \gets \text{setInitialValues}()$, for any $k\in[K]$, $c_k=1$ \Comment{Initialization}

%\BlankLine
\Comment{First estimates}
\While{$i \leq n_{\text{init}}$}{In each Local Server
\For{$1\leq s \leq S$ } {
    For any $k\in[K]$, $\bm{z}^s_k\gets \bm{z}^s_k \times c_k $  \;
    $\bm{z}^s \gets \text{CSC}\left( \bm{x}^s, \bm{z}^s, \bm{\Phi} \right)$\;
    For any $k\in[K]$, $\tilde{\phi}_k^s \gets \sum_{j=1}^{p_s^k} y^s_k\left[j\right]/p_s^k$  (Individual Barycenter) and
    $\tilde{\phi}_k^s$ is sent to the global server}
    [Central Server]\;
    For any $k\in[K]$, $\phi_k \gets \sum_{s=1}^S\tilde{\phi}_k^s/S  $ (Global Barycenter),
    $\phi_k,c_k \gets \phi_k/|\phi_k|,|\phi_k| $ (Normalization),
    and $\bm{\Phi}=(\phi_k)_{k\in[K]},(c_k)_k$ are sent to each local server. 
}

\BlankLine
\Comment{Personalization}
\While{$i \leq n_{\text{perso}}$}{In each Local Server
    \For{$1\leq s \leq S$ }{
    $\bm{z}^s \gets \text{CSC}\left( \bm{x}^s, \bm{z}^s, \bm{\Phi} \right)$\;
    $\bm{A}_s \gets \text{IPU}\left(\bm{X}_s, \bm{z}^s, \bm{\Phi},\bm{A}_s, f \right)$;
    For any $k\in[K]$, $\tilde{\Sigma}_s^k\gets p_s^k \, \mathcal{L}\left( \bm{a}^s_k \right)^\top \, \mathcal{L}\left( \bm{a}^s_k \right)$, $\tilde{\phi}_k^s \gets \sum_{j=1}^{p_s^k} \mathcal{L}\left( \bm{a}^s_k \right)^\top y^s_k\left[j\right]$ and
    $p_s^k,\bm{a}^s_k,\tilde{\phi}_k^s$ is sent to the global server}
    [Central Server]
    
    For any $k\in[K]$, $\tilde{\Sigma}_k\gets \sum_{s=1}^S p_s^k \, \mathcal{L}\left( \bm{a}^s_k \right)^\top \, \mathcal{L}\left( \bm{a}^s_k \right)$,
    $\phi_k \gets \tilde{\Sigma}_k^{-1}\sum_{s=1}^S \tilde{\phi}_k^s $ (Weighted Global Barycenter),
    
    $\phi_k,c_k \gets \phi_k/|\phi_k|,|\phi_k| $ (Normalization),
    and $\bm{\Phi}=(\phi_k)_{k\in[K]},(c_k)_k$ are sent to each local server.
}
\end{algorithm*}

\paragraph{Robustness improvement}
The personalization step might not be robust when the individual parameters are poorly estimated (recall that the update formula of the personalized common atom in \eqref{eq:MLE_Phi} assumes a perfect estimation of the individual parameters).
To increase its robustness, we follow \citep{mcmahan2017communication} to replace the update on the Local Servers by a gradient descent on both the individual parameters $\mathbf{a}^s$ and the personalized dictionaries $\tilde{\bm{\phi}}_k^s$ at the same time. 
The common dictionary can then be updated as $\bm{\phi}_k\gets \sum_{s=1}^S\tilde{\bm{\phi}}_k^s/S $ for any $k\in[K]$. 

In order to avoid over-adaptation of the personalized dictionaries, the individual parameters should be initialized with a preliminary step of IPU before running the joint update on each Local Server.
Such amendments result in a final complexity per local server asymptotically identical as the one of \Cref{alg:Federated_learning}, with an even smaller communication cost.

%%%%%%%%%%%%%%%%%%%%%%%%%%%%%%%%%%%%%%%%%%%%%%%%%%%%%%%%%%%%
\newpage
\section{ADDITIONAL NUMERICAL EXPERIMENTS}
\label{sec:numerical_details}
Unless otherwise stated, all experiments were conducted on a computing server equipped with an Intel\textsuperscript{\textregistered} Xeon\textsuperscript{\textregistered} Gold 5220R processor at $2.20 \, GHz$. 
The code used for the experiments can be found at the following online repository\footnote{\url{https://github.com/axelroques/PerCDL}.}.

As is customary in the field of time series representation~\citep{dupre2018multivariate, power2023using, liang2024personalized, jas2017learning, moreau2018dicod, morup2008shift, shaheen2017constrained, song2018spike}, no data separation into training/validation/test sets was implemented.
The ambition was to reconstruct a given dataset as truthfully as possible rather than investigate the generalization capabilities of the method (within-signal generalization is somewhat assumed implicitly due to the structured nature of the physiological activity recorded).

\paragraph{Interesting Implementations Detail.}
We have observed that slightly centering the common atoms after the IPU step by considering $\phi_{new}=\mathcal{T}_{\bar{\psi},\sigma}\phi_{old}$ with $\bar{\psi}=\sum_{i=1}^S \psi_{a_i}$ has a regularizing effect: it smooths the common atom. 
This optional additional step was employed in the experiments.
%However, centering has a drawback: it reduces the amplitude of peaks. 
% To remove this drawback, we could have compute a dtw-barycenter from $(y_k^s[j])_{s\in[S],j\in[p_s^k]}$ starting with the common atom found by PerCDL as initialization in order to improve a little the Figures for real data (it increases the peak amplitude). This last step used for visualization doesn't increase the complexity of the Algorithm.

\paragraph{Time Complexity.}
In order to give a better idea of the temporal complexity of the different methods described in this work, the next paragraph details the actual computation times observed on a personal laptops (MacBook Pro, M3 Max chip) for a subset of the experiments presented in this section.
\begin{enumerate}
    \item Synthetic dataset ($K=2$, $N=500$, $n_{steps}=5$):
    \begin{enumerate}
        \item $S=2$; IndCDL: $0.29s$; PopCDL: $0.17s$; PerCDL: $5.86s$.
        \item $S=16$: IndCDL: $2.20s$; PopCDL: $0.32s$; PerCDL: $7.33s$.
        \item $S=128$: IndCDL: $17.57s$; PopCDL: $2.79s$; PerCDL: $21.50s$.
        \item $S=1024$: IndCDL: $138.57s$; PopCDL: $18.46s$; PerCDL: $127.25s$.
    \end{enumerate}
    \item Gait dataset ($S=150$, $N=15457$, $n_{steps}=5$); IndCDL: $26.86s$; PopCDL: $5.94s$; PerCDL: $48.37s$.
    \item ECG dataset ($S=4465$, $N=1000$, $n_{steps}=20$): IndCDL: $2465.09s$; PopCDL: $292.67s$; PerCDL: $1731.81s$.
\end{enumerate}

%%%%%%%%%%%%%%%%%%%%%%%%%%%
\subsection{Validation on Synthetic Data}
An extensive list of experiments were conducted on synthetic data to verify the theoretical findings presented above. 
The toy dataset consisted of $S$ time series of equal-length $N$ containing $r=3$ repetitions of $K=2$ common pattern of length $L$, distributed randomly within each time series under the no overlap condition. 
The shared structure was transformed using random personalization parameters drawn uniformly random from $(-1, 1)^M$ and projected onto $\Theta$ in each time series with the transformation function described in section \ref{sec:transformation_function}. 
In the remainder of this section, the personalization hyper-parameters were fixed with $D=3$ and $W=10$. 
The sparsity balance parameter was fixed at $\lambda = 10^{-2}$.

%%%%%%%%%%%%%%
\paragraph{Finding the Activations.}
In this section, the CSC step of PerCDL was evaluated. 
PerCDL was initialized with the true common atoms and personalization parameters, which remained fixed for the duration of the experiment. 
\Cref{fig:app:activations} presents the results of the optimization process.

\begin{figure*}[!hbt]
    \centering
    \includegraphics[width=0.8\textwidth]{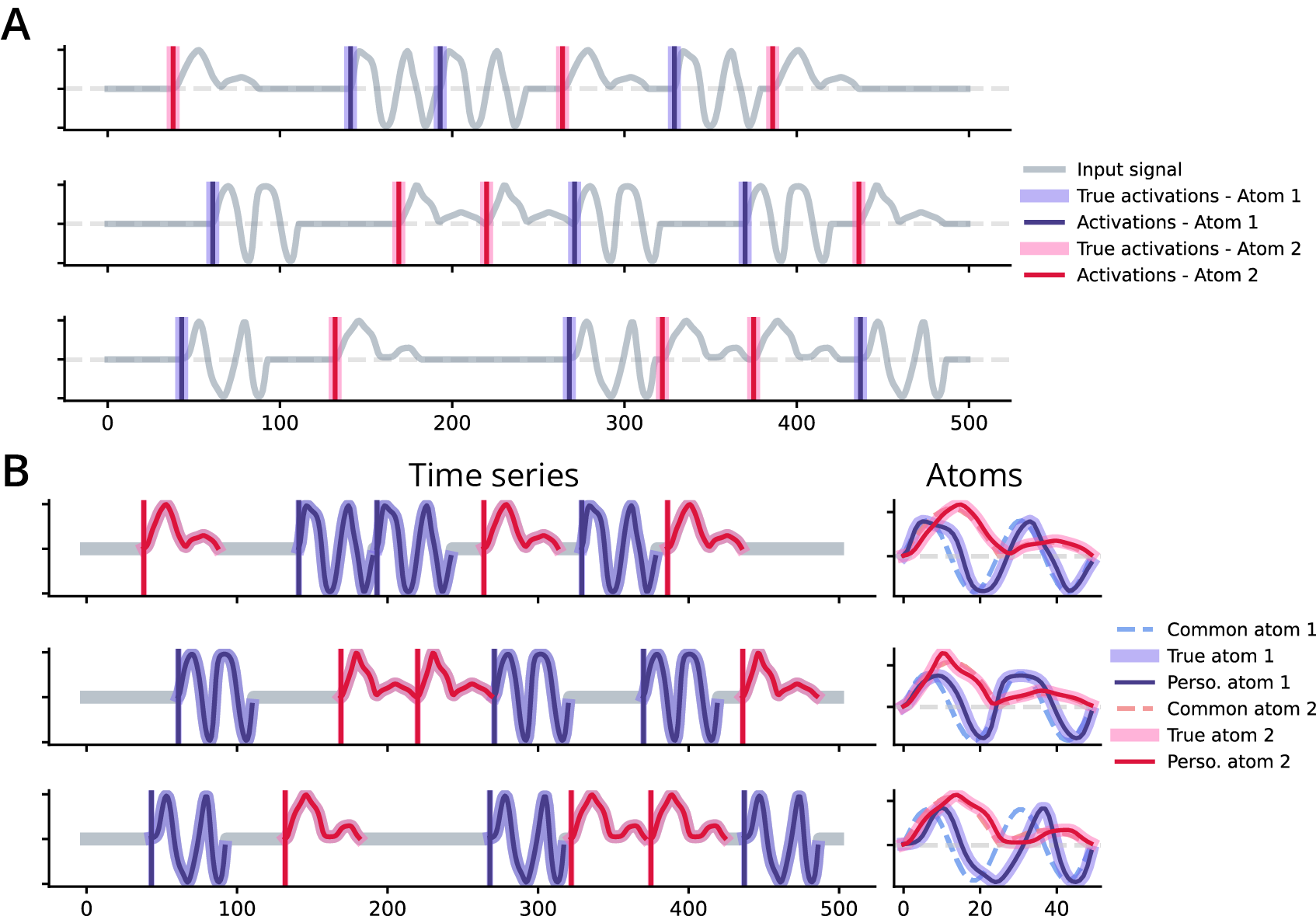}
    \caption{
    \textbf{Recovery of the Activations with Fixed Common Atoms and Personalization Parameters}. 
    Starting from known common atoms and personalizations parameters, PerCDL successfully recovers the true activations (A, solid and transparent vertical lines).
    Reconstruction of the dataset, as well as the common and personalized dictionaries are shown in B.
    Dataset parameters: $N=500$, $S=3$, $L=50$, $D=3$, $W=10$.
    Optimization parameters: $n_{steps}=250$, step size of $10^{-3}$. 
    }
    \label{fig:app:activations}
\end{figure*}

Our findings suggest that PerCDL can effectively recover a set of known activations when the common atoms and the personalization parameters are fixed.

%%%%%%%%%%%%%%
\paragraph{Finding the Common Atom.}
In this section, the CDU step of PerCDL was evaluated. PerCDL was initialized with the true activations and personalization parameters, which remained fixed for the duration of the experiment.
The common dictionary was initialized with the personalized atoms of the first signal. 
\Cref{fig:app:validation_common} presents the results of the optimization process.

\begin{figure*}[!hbt]
    \centering
    \includegraphics[width=0.9\textwidth]{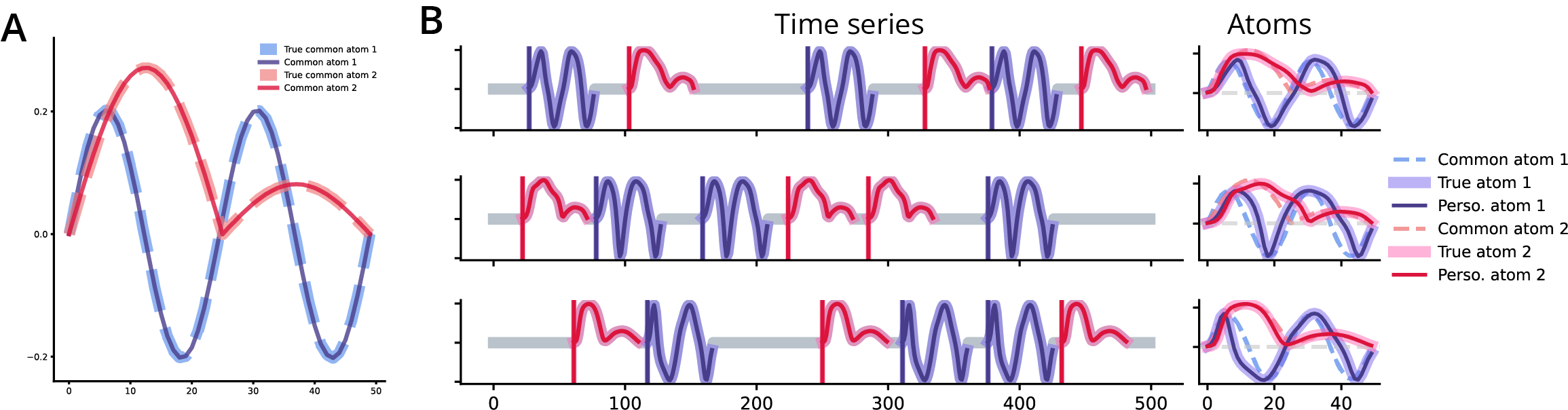}
    \caption{
    \textbf{Recovery of the Common Atoms with Fixed Activations and Personalization Parameters}. 
    Starting from known activations and personalization parameters, PerCDL successfully recovers the true common atoms (A, solid and dashed colored lines). 
    Reconstruction of the dataset, as well as the common and personalized dictionaries are shown in B.
    Dataset parameters: $N=500$, $S=3$, $L=50$, $D=3$, $W=10$.
    Optimization parameters: $n_{steps}=250$, step size of $10^{-3}$. 
    }
    \label{fig:app:validation_common}
\end{figure*}

Our findings suggest that PerCDL can effectively recover a set of known common atoms when the activations and the personalization parameters are fixed. 

%%%%%%%%%%%%%%
\paragraph{Finding the Personalization Parameters.}
In this section, the IPU step of PerCDL was evaluated. 
PerCDL was initialized with the true activations and common atoms, which remained fixed for the duration of the experiment.
\Cref{fig:app:perso} presents the results of the optimization process.

\begin{figure*}[!hbt]
    \centering
    \includegraphics[width=0.8\textwidth]{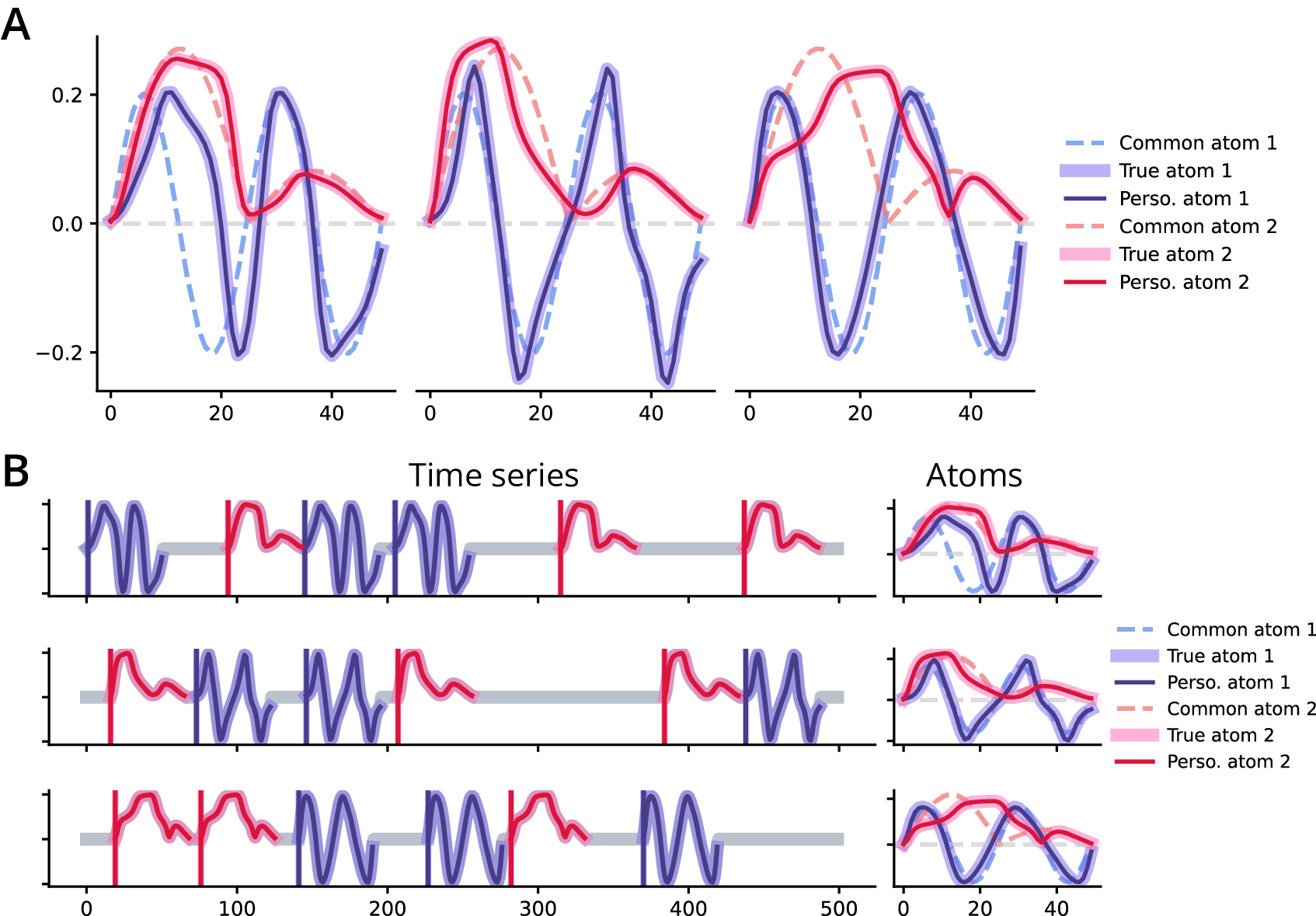}
    \caption{
    \textbf{Recovery of the personalization Parameters with Fixed Activations and Common Atoms}. 
    Starting from known activations and common atoms, PerCDL successfully recovers the true personalization parameters (A, solid and dashed colored lines). 
    Reconstruction of the dataset, as well as the common and personalized dictionaries are shown in B.
    Dataset parameters: $N=500$, $S=3$, $L=50$, $D=3$, $W=10$.
    Optimization parameters: $n_{steps}=250$, step size of $10^{-3}$. 
    }
    \label{fig:app:perso}
\end{figure*}

Our findings suggest that PerCDL can effectively recover a set of known personalization parameters when the activations and the common atoms are fixed.

%%%%%%%%%%%%%%%%%%%%%%%%%%%%%%%%%%%%%%%%%%%%%%%%%%%%%%%%%%%%
\subsection{Sensitivity Analysis}
\label{sec:sensitivity_analysis}
The sensitivity of the transformation $f$ \eqref{eq:transformation_warping} with respect to its hyper-parameters $D, W$ was investigated. 
A synthetic dataset was generated ($N=500$, $S=3$) with $(D,W)=(3,10)$ by applying a transformation $f(\cdot,\bm{a})$ on the two fixed common atoms shown in \cref{fig:app:validation_common}, with random parameters $\bm{a}$ uniformly sampled in the interval $[-1,1)^M$. 
PerCDL was then initialized with the known common atoms and fixed activations, and the effect of both hyper-parameters on the reconstruction error was studied using a grid search. 
Results are presented in \Cref{fig:sensitivity_analysis}.

\begin{figure}[!htb]
  \begin{minipage}[c]{0.5\textwidth}
  	\centering
    \includegraphics[width=0.95\textwidth]{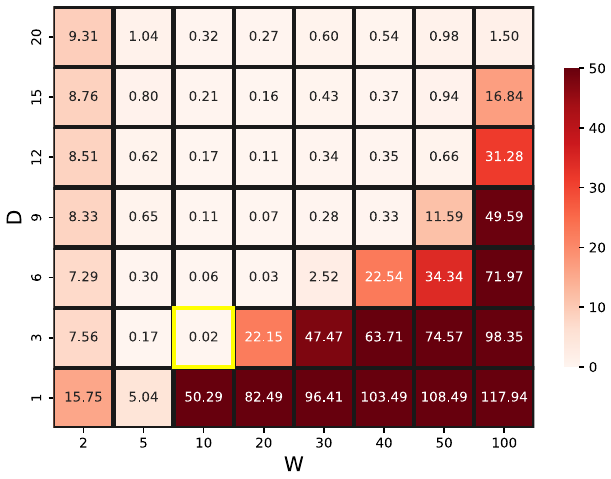}
  \end{minipage}\hfill
  \begin{minipage}[c]{0.5\textwidth}
    \caption{
    \textbf{Sensitivity Analysis}.
    Reconstruction error (Euclidean distance) versus the choice of time warping hyper-parameters $D, W$ on a synthetic dataset.
    The correct hyper-parameters used to generate the data are $(D,W)=(3,10)$, marked with a yellow box.
    }
    \label{fig:sensitivity_analysis}
  \end{minipage}
\end{figure}

Our results indicate that larger $D$ reduces the sensitivity to $W$, and generally leads to good performance.
This is unsurprising given the known benefits of over-parameterization~\cite{zhang2023going}.

%%%%%%%
\subsection{Convergence Toward the Common Atoms}
\label{sec:app:conv}
The ability of the different methods in identifying the common structures was evaluated on synthetic datasets. 
All datasets tested contained a varying number of time series, from $S=1$ to $S=1024$, of equal length $N=1000$, with $r=3$ repetitions of $K=2$ common atoms of length $L=50$, without overlap. 
Quantification results were presented in the main text. 
\Cref{fig:supp:synth_conv} illustrates qualitatively the convergence toward the common atoms with an increase in the number of signals. 

\begin{figure}[!hbt]
\centering
  \includegraphics[width=.8\linewidth]{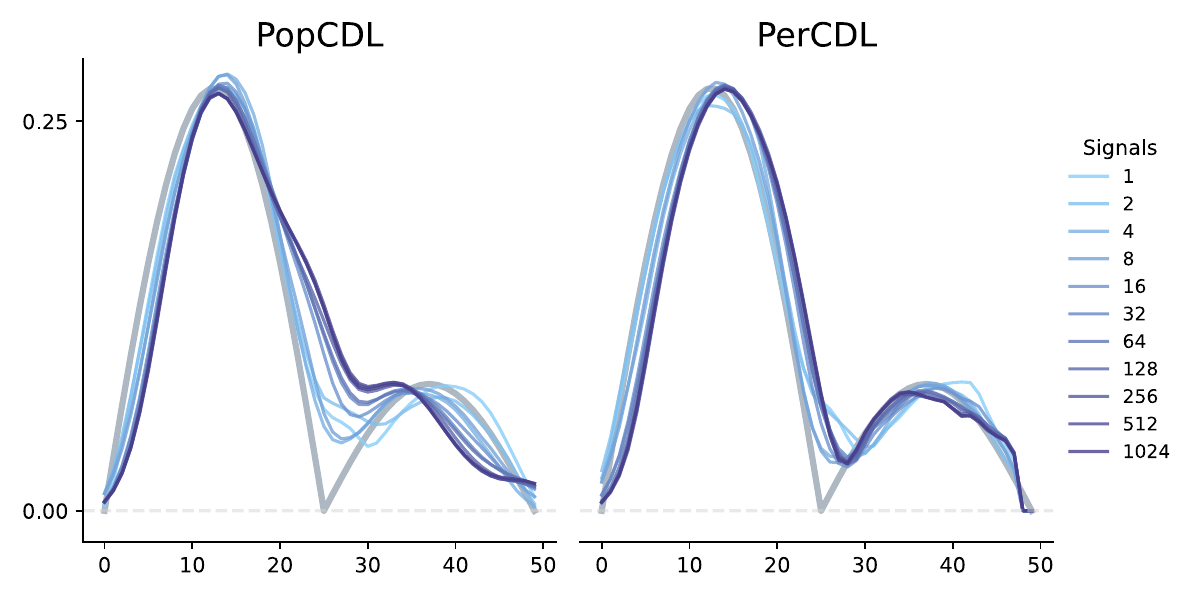}
\caption{
    \textbf{Recovery of the Common Atoms in Synthetic Data as a Function of $S$}. 
}
\label{fig:supp:synth_conv}
\end{figure}

PerCDL efficiently recovers the shared information with an increase in the dataset's size, contrary to PopCDL.

%%%%%%%

\subsection{Robustness to Noise}
\label{sec:robustness_noise}
The robustness of the different methods was evaluated on synthetic data contaminated with Gaussian or impulse noise. 
All synthetic datasets tested contained $S=32$ times series of equal length $N=1000$, with $r=3$ repetitions of $K=2$ common atoms of length $L=50$, without overlap. 
Results with impulse noise were presented in the main text. The following paragraph will investigate the impact of Gaussian noise. 

\Cref{fig:app:gaussian_noise}(A) presents the evolution of the distance to the personalized atoms as a function of the signal to Gaussian noise ratio (SNR). 
Two distance metrics were tested: the Euclidean distance and the DTW. 
Only the latter was included to ensure a fair comparison between methods despite potential discrepancies in temporal offset between the true and the learned atoms, which may arise due to different activation matrices. 
We observe that no matter the metric or the SNR level, the atoms learned through PopCDL and PerCDL are closer to the true underlying atoms compared to IndCDL under noisy conditions (SNR$<20$).

\Cref{fig:app:gaussian_noise}(B) presents the evolution of the reconstruction error---computed as the Euclidean distance between the input signal and the reconstruction using the activations and atoms identified following the optimization process---as a function of the SNR. 
Our results indicate that both PopCDL and PerCDL degrade relatively less than IndCDL under high noise levels ($SNR<20$).

\begin{figure}[!hbt]
\centering
\subfloat[\textbf{Distance to the personalized atoms}]{
  \includegraphics[width=.4\linewidth]{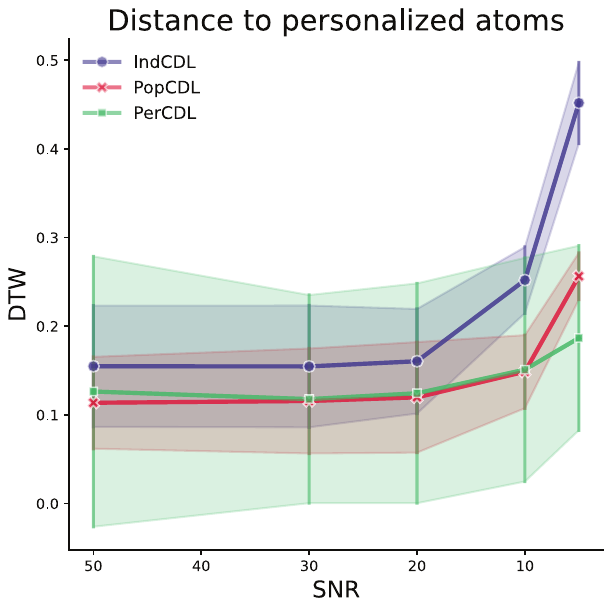}
}\hspace{1em}
\subfloat[\textbf{Reconstruction error}]{
  \includegraphics[width=.4\linewidth]{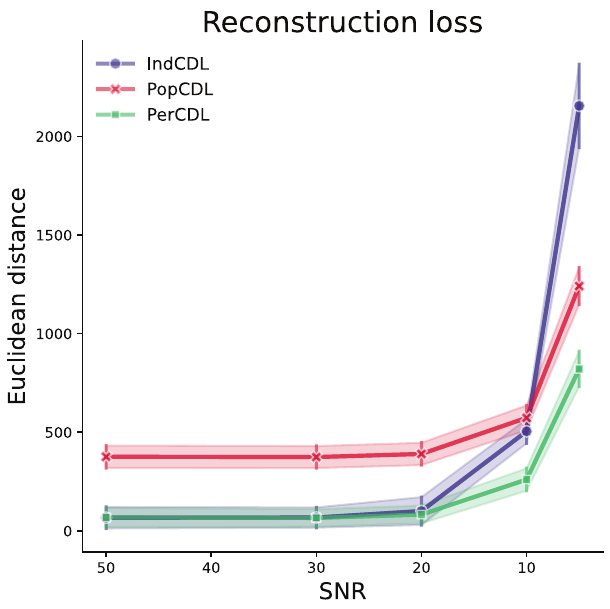}
} 
\caption{
    \textbf{Robustness of the Different Methods in the Presence of Gaussian Noise}. 
    Personalization hyper-parameters: $D = 3$, $W = 10$.
}
\label{fig:app:gaussian_noise}
\end{figure}

%%%%%%%%%%%%%%%%%%%%%%%%%%%%%%%%%%%%%%%%%%%%%%%%%%%%%%%%%%%%
\subsection{Application to Locomotion Data}
\label{sec:app:locomotion}
The interested reader may find the complete dataset used for the experiment at the following online repository\footnote{\url{https://github.com/deepcharles/gait-data}.} and a detailed description of the experiment conducted in the original article~\citep{truong2019}. 
Briefly, three groups of participants (Healthy, Orthopedic, and Neurological) underwent the same basic walking protocol while wearing an inertial measurement unit (IMU) on the dorsal face of their foot. 
Two brands of IMUs were employed (XSens™: $100$ $Hz$, device dimensions $47 \times 30 \times 13$ $mm$; Technoconcept\textregistered: $100$ $Hz$, device dimensions $49 \times 38  \times 19$ $mm$). 
Participants in the orthopedic group suffered either from lower limb osteoarthrosis or from cruciate ligament injury. 
The neurological group was composed of four specific pathologies: hemispheric stroke, Parkinson’s disease, toxic peripheral neuropathy and radiation-induced leukoencephalopathy.

During testing, $50$ participants from each group were chosen at random. 
For each, we focused specifically on the angular velocity of the right foot in the sagittal plane, as is customary in gait analysis~\citep{voisard2023automatic}. 
This resulted in a dataset of size $S=150$, with $N=15457$.
A single common atom ($K=1$) was searched for, of size set to $L=100$ to match the length of the gait cycle.
Personalization hyper-parameters were set to $D=4$, $W=15$ following a brief fine-tuning process.
Optimization parameters were left to default, i.e., $n_{\text{steps}}=25$ with a step size of $10^{-3}$, and $\lambda = 10^{-2}$.
The common dictionary was initialized with the gait cycle of a random participant from the Healthy group.

%%%% 

\paragraph{Gait Segmentation.}
\label{sec:app:gait_segmentation}
PerCDL's gait segmentation performance (i.e., the location of its activations) was compared to the state-of-the-art gait segmentation method \citep{voisard2023automatic}. 
Briefly, \citeauthor{voisard2023automatic} used a sophisticated template-based matching algorithm to identify gait cycles event in accelerometry signals. 
To compare the two methods in a fair manner, we assessed whether they could identify the start of a gait cycle, with a tolerance margin set relatively high ($\tau_{\text{tol}}=100$ samples or roughly one second) as the precise beginning of a gait cycle is an ambiguous concept. 
The number of true positives, i.e., the number of correctly identified gait cycles, the number of false negatives, i.e., the number of gait cycles that were undetected, and the number of false positives, i.e., the number of wrong detection, were computed. 
The dataset contained expert annotations which acted as a ground truth. 
Finally, we observed that PerCDL was sometimes too sensitive to small movements of the body preceding and succeeding locomotion: to prevent bias in the results, activations with an absolute value smaller than $\tau_{\text{act}}=2\%$ of the average of non-zero values were removed. 
Results are presented in \Cref{tab:app:gait_segmentation}. 

\begin{table*}[!hbt]
\centering
\caption{
    \textbf{Segmentation Performance of PerCDL vs. SOTA}. 
    The gait signals of $150$ participants were used. 
    True positive rate: rate of correctly identified gait cycles. 
    False negatives rate: rate of undetected gait cycles. 
    Proportion of false positives: number of wrong detection divided by the total number of gait cycles.
}
\begin{tabular}{l|ll}
                                       & \textbf{Voisard et al.} & \textbf{PerCDL} \\ \hline
\textbf{True positives rate}           & 0.959                   & 0.988           \\
\textbf{False negatives rate}          & 0.040                   & 0.012           \\
\textbf{Proportion of false positives} & 0.004                   & 0.033          
\end{tabular}
\label{tab:app:gait_segmentation}
\end{table*}

Our findings suggest that PerCDL is competitive with the state-of-the-art in detecting the gait cycles. 
In the dataset evaluated, it offers a better detection than \citet{voisard2023automatic}, although with a slight increase in the number of false detection. 
Note that this effect could be improved by fine-tuning the activation threshold $\tau_{\text{act}}=2\%$, which was not done in this simple experiment.

%%%% 

\paragraph{IndCDL vs. PopCDL vs. PerCDL.}
This section compares the three methods on the gait dataset. 
Initialization was identical for IndCDL, PopCDL and PerCDL. 
Normalized reconstruction errors for the different methods are as follows:
$\text{IndCDL} = 0.055$, $\text{PopCDL} = 0.109$, $\text{PerCDL} = 0.063$, $\text{IndCDL barycenter} = 0.116$.
IndCDL's barycenter corresponds to the reconstruction of the dataset with the Euclidean barycenter of IndCDL's results. 

Because IndCDL essentially over-fits each signal, it presents the smallest reconstruction error. 
However, it fails to capture the global structure, as a reconstruction of the dataset with the averaged atom identified with IndCDL yields the highest reconstruction error (referred to as IndCDL's barycenter). 
By construction, PopCDL captures the shared structure but lacks reconstruction power. 
Finally, PerCDL builds from the shared structured to fit each signal specifically, thus maintaining both global and local information, i.e., harnessing the structure representation capabilities from PopCDL while reaching reconstruction results similar to IndCDL.  

In \Cref{fig:IndCDL_PopCDL}, we observe that the average of IndCDL individual atoms and the common atoms of PopCDL are similar.
\begin{figure*}[!hbt]
\centering
  \begin{tabular}[b]{c}
    \includegraphics[width=.29\linewidth]{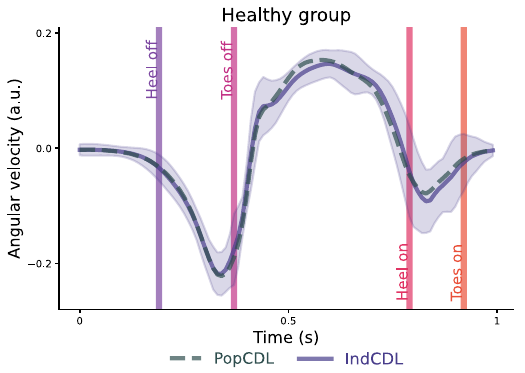} \\
    \small (a) Healthy group.
  \end{tabular} \qquad
  \begin{tabular}[b]{c}
    \includegraphics[width=.29\linewidth]{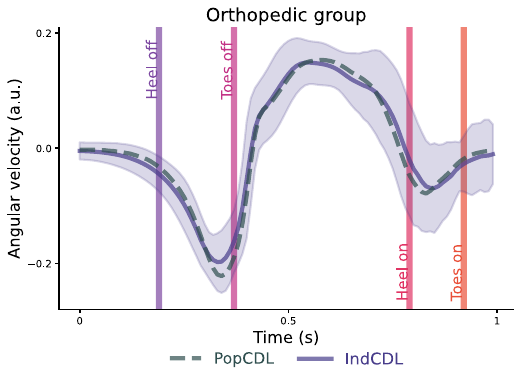} \\
    \small (b) Orthopedic group.
  \end{tabular}
  \begin{tabular}[b]{c}
    \includegraphics[width=.29\linewidth]{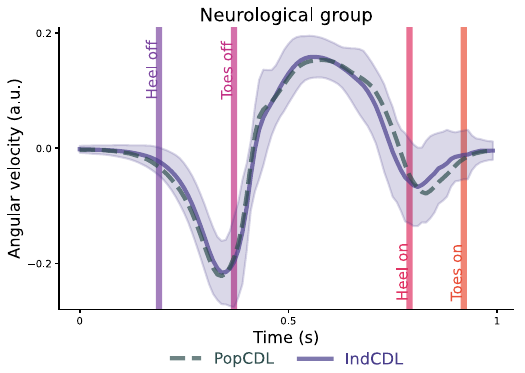} \\
    \small (c) Neurological group.
  \end{tabular} \qquad
\caption{
    \textbf{Comparison of IndCDL and PopCDL on the Healthy, Orthopedic and Neurological Groups}. 
    The shaded region corresponds to $2$ standard deviations of the individual atoms returned by IndCDL, and the vertical colored lines timestamp the different phases of the gait cycle. 
    The solid lines of IndCDL are its Euclidean barycenter on a given group (Healthy, Orthopedic, Neurological).
}
\label{fig:IndCDL_PopCDL}
\end{figure*}

%%%% 

\paragraph{IndCDL.}
Let us try to conduct a similar analysis as the one presented in \Cref{sec:experiments:gait_analysis} with IndCDL. 
Initialization was identical to that of PerCDL. 
\Cref{fig:app:gait_comp_indcdl} presents the group variability analysis of the atoms identified. 

\begin{figure}[!hbt]
\centering
  \includegraphics[width=.8\linewidth]{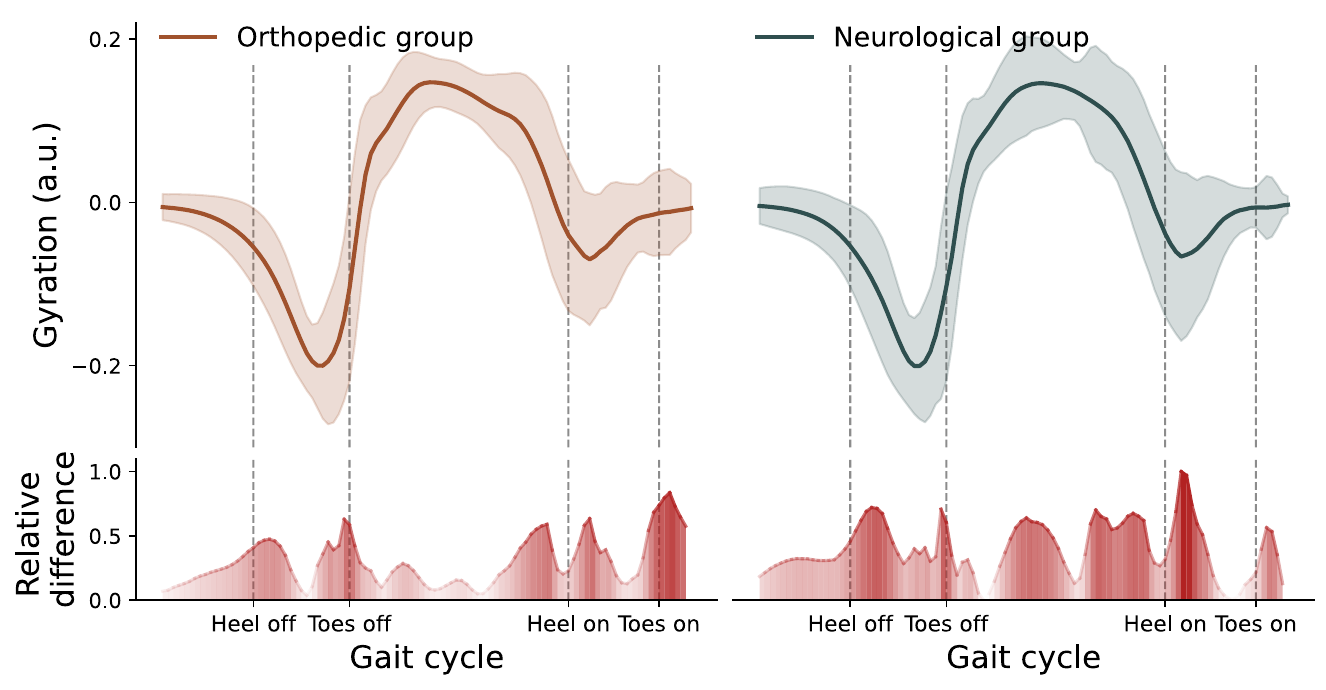}
\caption{
    \textbf{Application of IndCDL to Locomotion Data}.
    Atoms learned by PerCDL in the Orthopedic (solid brown line) and Neurological (solid green curve) groups, with two standard deviations (shaded areas). Below each gait cycle, the red-colored curve represents the relative variability in the gait cycle for the current population compared with the Healthy group.
}
\label{fig:app:gait_comp_indcdl}
\end{figure}

A striking result is the presence of ``parasitic variability'' at the start and end of each atom. 
This variability is uninformative of the gait cycle, and simply corresponds to an overfit of each signal caused by IndCDL's learning scheme. 
This observation suggests that PerCDL is better suited than IndCDL for this kind of analysis.

%%%% 

\paragraph{Group-level analysis with PopCDL.}
When explicit groups can be identified in a dataset, one may be tempted to apply CDL to each group independently.
This reduces to applying PopCDL to each sub-population in the data separately and should intuitively yield a result somewhere between IndCDL and PopCDL. 
To test this idea, we ran PopCDL on each group from the gait dataset (\Cref{tab:app:group_popcdl}).

\begin{table}[!htb]
\centering
\caption{
    \textbf{Group-level Analysis with PopCDL}.
    The values correspond to the DTW distance between the common atoms from the two groups estimated independently using PopCDL.
}
\begin{tabular}{l|l|l|l}
 & \textbf{Healthy - Orthopedic} & \textbf{Healthy - Neurological} & \textbf{Orthopedic - Neurological} \\ \hline
\textbf{DTW distance} & 0.0531 & 0.0542 & 0.0367
\end{tabular}
\label{tab:app:group_popcdl}
\end{table}

Our results suggest that the common structures found by PopCDL on the different groups are highly similar (similar distance between Healthy vs. Orthopedic and Healthy vs. Neurological, and small distance between Orthopedic and Neurological).
This emphasizes the idea that in the context of physiological, it is often the variation around the common structures that gives interesting information about an individual; something that only PerCDL can capture, not PopCDL.

%%%% 

\paragraph{PerCDL's Personalized Atoms.}
The personalized atoms found by PerCDL were investigated. 
\Cref{fig:app:gait_personalization} shows an example of the personalized atom found in a single participant per pathology group. 

\begin{figure*}[!hbt]
\centering
  \begin{tabular}[b]{r}
    \includegraphics[width=.45\linewidth]{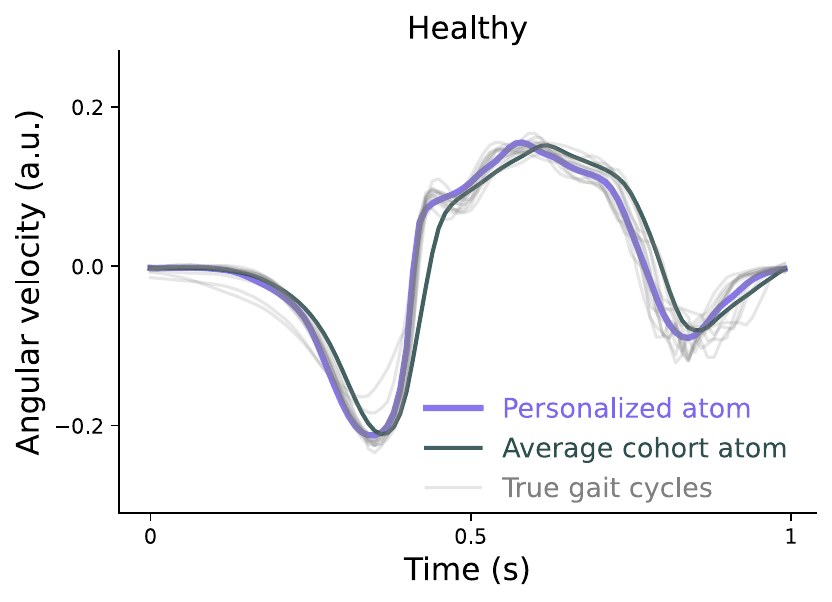} \\
    \small (a) Participant from the Healthy group.
  \end{tabular} \qquad
  \begin{tabular}[b]{l}
    \includegraphics[width=.45\linewidth]{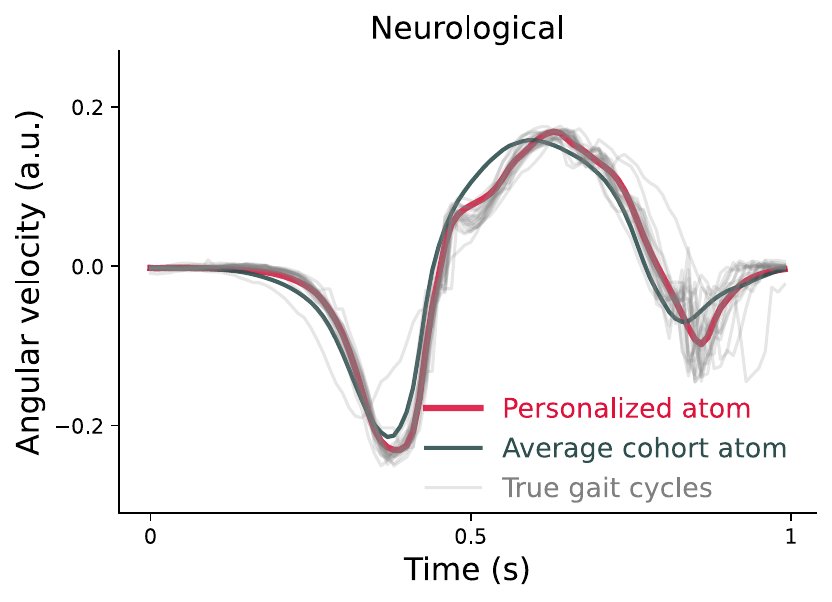} \\
    \small (b) Participant from the Neurological group.
  \end{tabular} \qquad
\caption{
    \textbf{Example of the Personalized Atom Found in a Single Participant per Pathology Group}. Personalized atoms identified by PerCDL (colored lines) in a single participant, average personalized atom per pathology group (bold black line), and all raw occurrences of the gait cycle in the participant's signal (light grey lines). PerCDL efficiently captures individual-specific variability.
}
\label{fig:app:gait_personalization}
\end{figure*}

We observe that the personalized atoms differ from the population's average and instead accurately match the raw gait cycles manually extracted from an expert's annotation. 
Thus, PerCDL is able to efficiently capture participant-specific variability in their gait cycle.

This results was further quantified in \Cref{tab:app:personalization}. 
As expected, the personalized atoms have the smallest distance to the true gait cycles, compared to the common atoms or the population average. 
This further demonstrates the efficiency of PerCDL's personalization step.
\begin{table*}[!hbt]
\centering
\caption{
    \textbf{Euclidean Distance Between the True Gait Cycles and the Atoms}. 
    The personalized atom corresponds to the personalized atom learned by PerCDL for the given subject. 
    The cohort personalization mean corresponds to the atom obtained after averaging all personalized atoms for the given pathology group. 
    The distance between the true cycles and the personalized atom is smaller than with all other atoms.
}
\begin{tabular}{ll|lll}
 &                                             & \textbf{Healthy} & \textbf{Orthopedic} & \textbf{Neurological} \\ \cline{2-5} 
 & True cycles vs. Personalized atom           & 0.160            & 0.143               & 0.184                 \\
 & True cycles vs. Cohort personalization mean & 0.305            & 0.274               & 0.302                 \\
 & True cycles vs. Common atom (PerCDL)        & 0.299            & 0.350               & 0.330                 \\
 & True cycles vs. Common atom (PopCDL)        & 0.192            & 0.481               & 0.554                
\end{tabular}
\label{tab:app:personalization}
\end{table*}

%%%%%%%%%%%%%%%%%%%%%%%%%%%%%%%%%%%%%%%%%%%%%%%%%%%%%%%%%%%%
\newpage
\subsection{Application to ECG Data}
\label{sec:app:ECG_analysis}

IndCDL, PopCDL and PerCDL were evaluated on a public ECG dataset\footnote{The dataset is available online: \url{https://doi.org/10.13026/kfzx-aw45}.} (PTB-XL; \citealp{wagner2020ptb}).
A subset of $S=2650$ signals, each of dimension $=12$ (due to the $12$ ECG leads: $I$, $II$, $III$, $AVL$, $AVR$, $AVF$, $V1-V6$) and sampled at $100$ $Hz$, were randomly selected. 

Prior to the experiments, the hyperparameters for PerCDL were briefly tuned on a separate set of $20$ total subjects. 
The selected parameters set is as follows: $D=4$, $W=7$, $n_{steps}=20$, $K=1$, $L=65$. 
All methods were initialized with activations set to zero and the common dictionary was populated with a single ECG motif from a random subject from the healthy group (superclass ``NORM''). 

%%%%%%%%%%%%%%
\paragraph{Atoms and Reconstruction.}
IndCDL, PopCDL and PerCDL were applied to the dataset.
\Cref{fig:app:ecg_atoms} presents the atoms identified by IndCDL and PerCDL for the $II$ derivation, while \Cref{fig:app:ecg_recon} illustrates the reconstruction of three $II$ signals with the different methods for patients in the ``Myocardial Infarction'' (MI) pathology group. 

\begin{figure*}[!hbt]
\centering
  \begin{tabular}[b]{c}
    \includegraphics[width=.95\linewidth]{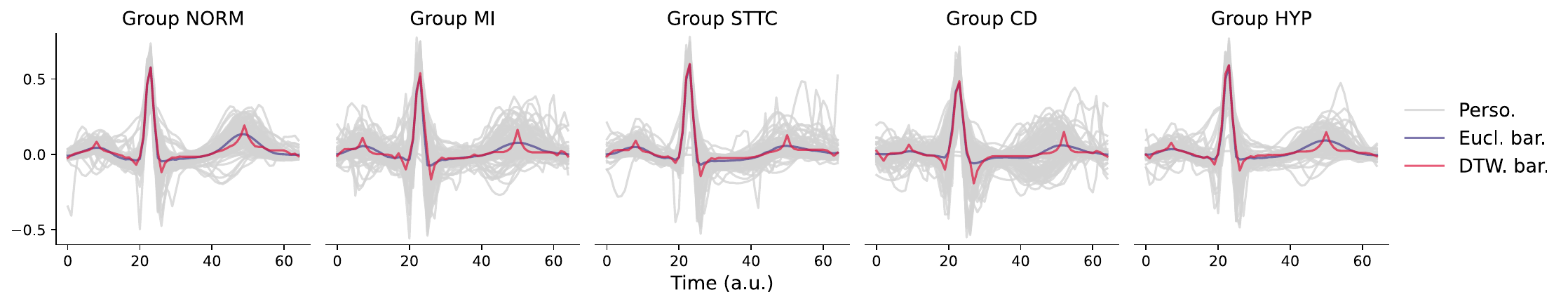} \\
    \small (a) Atoms identified with IndCDL.
  \end{tabular} \qquad
  \begin{tabular}[b]{c}
    \includegraphics[width=.95\linewidth]{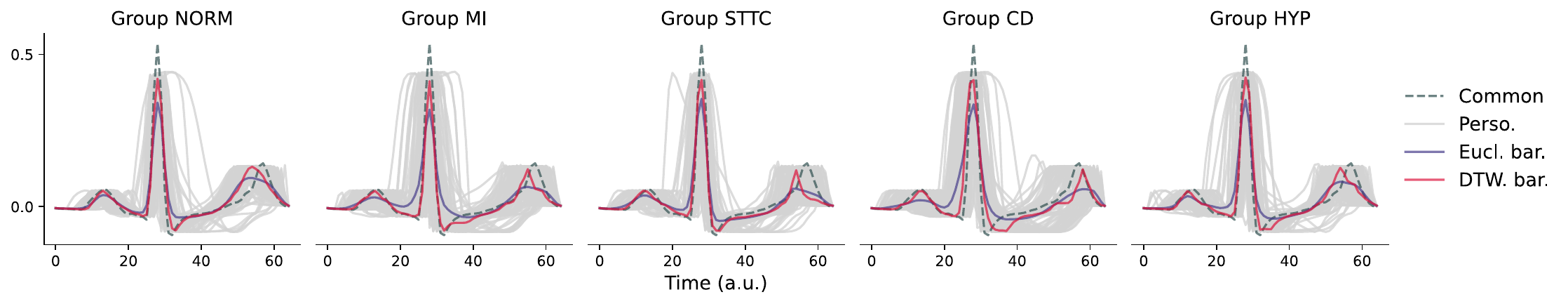} \\
    \small (b) Personalized atoms identified with PerCDL.
  \end{tabular} \qquad
\caption{
    \textbf{Atoms Identified by IndCDL and PerCDL in the Different Pathology Superclasses}. 
    (grey) Subject-specific atom. 
    (blue) Euclidean barycenter of all individual atoms. (red) DTW barycenter of all individual atoms.
}
\label{fig:app:ecg_atoms}
\end{figure*}

\begin{figure*}[!hbt]
\centering
  \begin{tabular}[b]{c}
    \includegraphics[width=.46\linewidth]{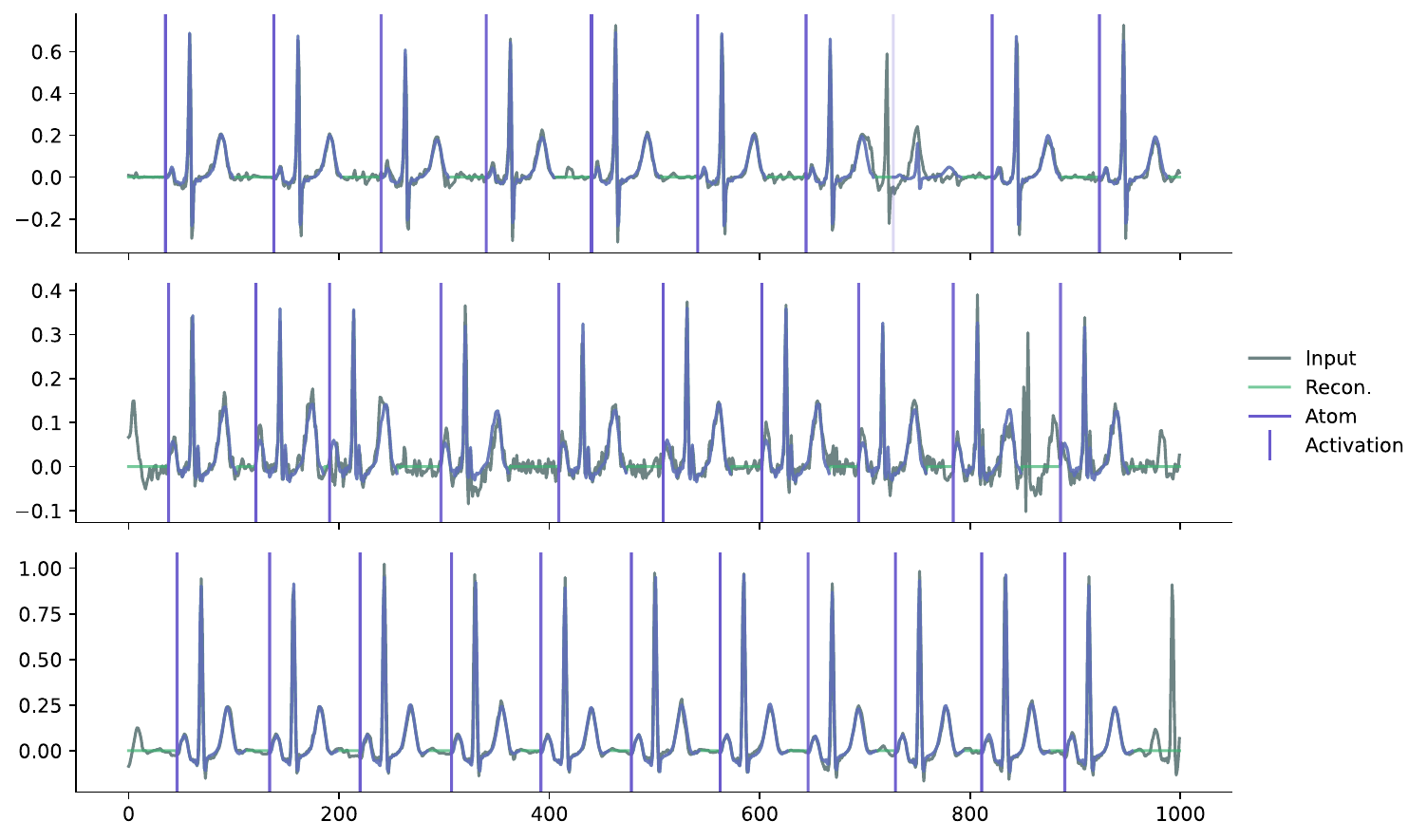} \\
    \small (a) Reconstruction with IndCDL.
  \end{tabular} 
  \begin{tabular}[b]{c}
    \includegraphics[width=.46\linewidth]{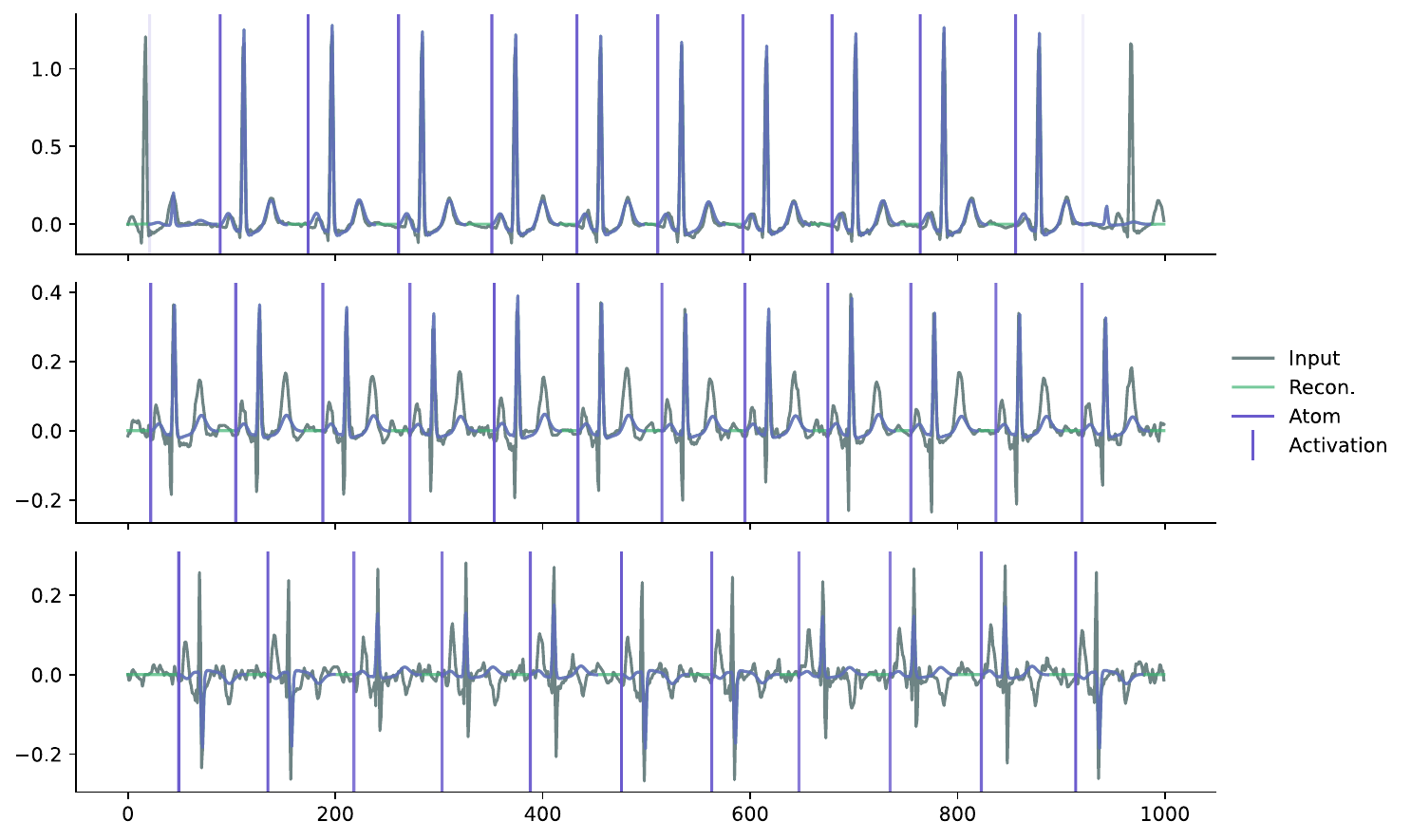} \\
    \small (b) Reconstruction with PopCDL.
  \end{tabular} 
  \begin{tabular}[b]{c}
    \includegraphics[width=.46\linewidth]{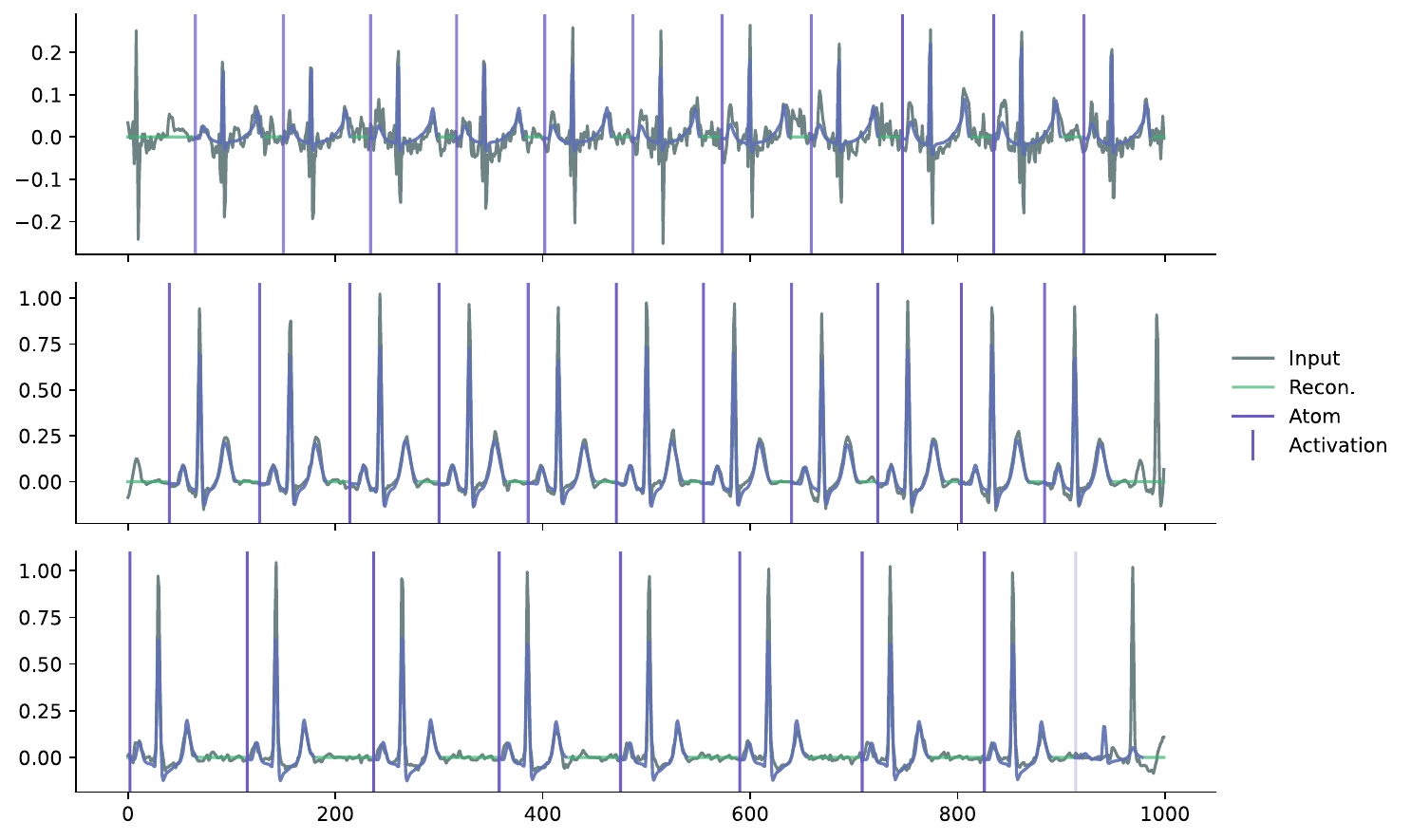} \\
    \small (c) Reconstruction with PerCDL.
  \end{tabular} 
\caption{
    \textbf{Reconstruction of Three Example Signals from Three Different Subjects in the MI Group}. 
    (grey) Input signals. 
    (green) Reconstruction. 
    (blue curves) Atoms. 
    (blue, vertical bars) Activations.
}
\label{fig:app:ecg_recon}
\end{figure*}

All methods were able to identify the ECG events efficiently. 
The dataset presented great variability within each superclass, as they contained multiple specific sub-pathologies. 
This explains the relatively cluttered plots for the atoms in \Cref{fig:app:ecg_atoms}. 
This effect is strikingly apparent for IndCDL, which has a tendency to overfit each signal, resulting in very ``volatile'' atoms that may contain noise (see \Cref{fig:app:ecg_recon}, a). 
Consequently, any effort to average all of the individual atoms is a recipe for failure.
On the complete opposite side of the spectrum, PopCDL only learns a common structure that cannot be used directly to model each signal, hence resulting in poor reconstructions. 
PerCDL constitutes an appealing middle ground, as it captures enough population-level structure to maintain a coherent shape across its atoms, while also capturing enough personal information to reconstruct efficiently each signal. On \Cref{fig:app:ecg_recon}, PerCDL appears to be more robust to noise than IndCDL, although at the price of smoother atoms.
Normalized reconstruction errors for the different methods are as follows:
$\text{IndCDL} = 0.176$, $\text{PopCDL} = 0.370$, $\text{PerCDL} = 0.282$, $\text{IndCDL barycenter} = 0.417$.

%%%%
\paragraph{Personalization Parameters.}
We then wondered whether the personalization parameters learned by PerCDL could efficiently extract meaningful information from the dataset.
To that end, we tested whether each signal could be classified into its corresponding pathology superclass ($5$-class classification task) using only the personalization matrix.

The parameters matrix was separated into train and test sets ($70\%-30\%$ split).
A simple SVM (\textit{scikit-learn} Python library, parameters $C=1$, gamma='scale', kernel='rbf') was then trained on the personalization parameters to classify each signal. 
We obtained an AUC score (macro, ovo) of $0.864$. 
This result is comparable to the result of some deep learning techniques, namely \textsc{Wavelet+NN} (\textit{cf.}, \citealp{strodthoff2020deep}) that had an AUC score of $0.874$.

In order to try and provide a baseline for comparison, we used a k-nearest neighbor classifier (scikit-learn Python library, default parameters) with the raw personalized patterns as inputs.
This method obtained an AUC score of $0.728$ with the Euclidean distance metric and $0.699$ for the DTW distance metric.
Thus, PerCDL's personalization parameters offer a better classification performance than using the personalized atoms directly. 
Although the primary goal of our paper is not classification-related, our findings suggest nonetheless that the temporal warps employed in PerCDL effectively capture pathology-specific features. 

All in all, our findings suggest that the personalization parameters discover relevant characteristics of the input signals.

\end{document}